\documentclass[10pt,lettersize,journal]{IEEEtran}

\usepackage{tikz}
\usepackage{amsmath,amsfonts,amssymb,amsthm}

\usepackage{algorithm,algpseudocode}
\usepackage{tikz}
\usepackage{bbm}
\usepackage{multirow}
\usepackage{fontawesome5,pifont} 
\usepackage{subcaption}
\usepackage{adjustbox}
\usepackage{environ}
\usepackage{resizegather}
\usepackage{mdframed}

\usepackage{url}

\newtheorem{Thm}{Theorem}[section]
\newtheorem{Lem}[Thm]{Lemma}

\newtheorem{Exm}[Thm]{Example}

\newcounter{protocol}


\newcounter{protocol*}
\makeatletter
\newenvironment{protocol*}[1][htb]{%
    \let\c@algorithm\c@protocol
    \renewcommand{\ALG@name}{Protocol}
    \begin{algorithm*}[#1]%
    }{\end{algorithm*}
}
\makeatother

\newcommand{\ceil}[1]{\left\lceil#1 \right\rceil}

\newcommand{\round}[1]{\left\lfloor#1 \right\rceil}
\newcommand{\abs}[1]{\left|#1 \right|}

\newcommand{\me}[2]{\widetilde{#1}\left(#2\right)}

\renewcommand{\bf}[1]{\ifmmode\mathbf{#1}\else\bfseries{#1}\fi}

\renewcommand{\tt}[1]{\ensuremath{\texttt{#1}}}
\newcommand{\com}[1]{\ensuremath{\tt{com}_{#1}}}

\newcommand{\negl}[1]{\ensuremath{\tt{negl}\left({#1}\right)}}

\renewcommand{\paragraph}[1]{\medskip\noindent\textbf{#1.}}

\def\F{\mathbb{F}}
\def\G{\mathbb{G}}

\def\1{\mathbbm{1}}

\newcommand{\case}[2][lllllllllllllllllllllllllllllllllllll]{\left\{\begin{array}{#1}#2 \\ \end{array}\right.}

\usepackage{ulem}

\begin{document}

\date{}

\title{zkDL: Efficient Zero-Knowledge Proofs of Deep Learning Training}
\author{Haochen Sun, Tonghe Bai, Jason Li, Hongyang Zhang}
\maketitle

\begin{abstract}
The recent advancements in deep learning have brought about significant changes in various aspects of people's lives. Meanwhile, these rapid developments have raised concerns about the legitimacy of the training process of deep neural networks. To protect the intellectual properties of AI developers, directly examining the training process by accessing the model parameters and training data is often prohibited for verifiers.

In response to this challenge, we present zero-knowledge deep learning (zkDL), an efficient zero-knowledge proof for deep learning training. To address the long-standing challenge of verifiable computations of non-linearities in deep learning training, we introduce zkReLU, a specialized proof for the ReLU activation and its backpropagation. zkReLU turns the disadvantage of non-arithmetic relations into an advantage, leading to the creation of FAC4DNN, our specialized arithmetic circuit design for modelling neural networks. This design aggregates the proofs over different layers and training steps, without being constrained by their sequential order in the training process.

With our new CUDA implementation that achieves full compatibility with the tensor structures and the aggregated proof design, zkDL enables the generation of complete and sound proofs in less than a second per batch update for an 8-layer neural network with 10M parameters and a batch size of 64, while provably ensuring the privacy of data and model parameters. To our best knowledge, we are not aware of any existing work on zero-knowledge proof of
deep learning \emph{training} that is scalable to million-size networks.
\end{abstract}

\section{Introduction}

The rapid development of deep learning has garnered unprecedented attention over the past decade. However, with these advancements, concerns about the legitimacy of deep learning training have also arisen. In March 2023, Italy became the first Western country to ban ChatGPT amid an investigation into potential violations of the European Union's General Data Protection Regulation (GDPR). Furthermore, in January 2023, Stable Diffusion, a prominent image-generative model, faced accusations from a group of artist representatives over the infringement of copyrights on millions of images in its training data. As governments continue to impose new regulatory requirements on increasingly advanced AI technologies, there is an urgent need to develop a protocol that verifies the legitimacy of the training process for deep learning models. However, due to intellectual property and business secret concerns, model owners are typically hesitant to disclose their proprietary training data or model snapshots for legitimacy investigations.

Despite considerable efforts in verifiable machine learning to address this dilemma, many fundamental questions remain unanswered. Currently, cryptography-based approaches have mainly focused on inference-time verification \cite{zkcnn, zen, vcnn, pvcnn, zkml, mystique, ezdps}, leaving training-time verification largely unexplored because of the significant computational demands and complex operations involved. Pioneering works in verifiable training \cite{veriml, unlearning, unlearning-2} have mostly concentrated on basic machine learning algorithms, such as linear regression, logistic regression, and SVMs. The deep learning algorithms explored are limited by the size of the neural networks that the proof system backend can efficiently handle, making them inapplicable to modern deep neural networks.

Additionally, the advancement of verifiable deep learning is hindered by non-arithmetic operations, notably activation functions such as ReLU. These functions, although prevalent in deep learning, are not intrinsically supported by zero-knowledge proof (ZKP) systems. Early studies in verifiable deep learning training have explored square activation and polynomial approximation as arithmetic alternatives to traditional activation functions \cite{veriml, unlearning, DBLP:journals/corr/abs-2011-05530}. Yet, these alternatives diverge from standard deep learning architectures, raising questions about the effectiveness of such models. In light of the present state of verifiable deep learning training, confronting the challenges posed by activation functions such as ReLU is crucial for practical applications.

Furthermore, the representation of neural networks and their training process as arithmetic circuits that can accommodate activation function management is still unclear. Notably, beyond the innate layered structure, the training process encompasses both forward and backward propagations across numerous training steps, amplifying the complexity of the system to be modelled. Thus, a modelling approach for the training process over the neural network that aligns with not only the tensor structures but also the layered architecture and the numerous training steps is pivotal to the formulation of an efficient ZKP scheme for deep learning training.

In response to these challenges, we introduce \emph{zkDL}, the first zero-knowledge proof for deep learning training. Through zkDL, model developers can provide assurances to regulatory bodies that the model has undergone proper training in accordance with the committed training data, consistent with the specified training logic, thereby addressing concerns about model legitimacy. Our principal contributions include:

\begin{itemize}
    \item We present \emph{zkReLU}, an efficient specialized zero-knowledge proof designed specifically for the \emph{exact} computation of the Rectified Linear Unit (ReLU) and its backpropagation. This is achieved without the need for polynomial approximations to manage non-arithmetic operations. The foundational structure of zkReLU also facilitates our innovative arithmetic circuit design to depict the entire training procedure.
    
    \item We introduce \emph{FAC4DNN}, a modelling scheme that represents the training process over deep neural networks as arithmetic circuits. FAC4DNN is acutely aware of the unique structures inherent in the entire training process, including both tensor-based and layer-based configurations, as well as the repeated execution of similar operations across multiple training steps. Astutely, FAC4DNN leverages the unavoidable alternations made in zkReLU, turning them into an advantage. This enables proofs for different layers and training steps to be aggregated, bypassing the traditional sequential processing as in the training process. As a result, there are both empirical and theoretical reductions in computational and communicational overheads when conducting the proof.
    
    \item In addition to pioneering zero-knowledge verifiability for real-world scale neural networks, we have implemented zkDL as the first zero-knowledge deep-learning framework using CUDA. Benefiting from the combined strengths of zkDL's design and implementation, we markedly advance toward practical zero-knowledge verifiable deep learning training for real-world industrial applications. Specifically, on an 8-layer network containing over 10 million parameters and a batch size of 64 using the CIFAR-10 dataset, we have confined the proof generation time to \emph{less than 1 second} per batch update.
\end{itemize}

\subsection{Overview of zkDL}

\begin{figure*}
    \centering
    \begin{subfigure}[t]{0.3\textwidth}
        \centering
        \includegraphics[page=3,width=\textwidth]{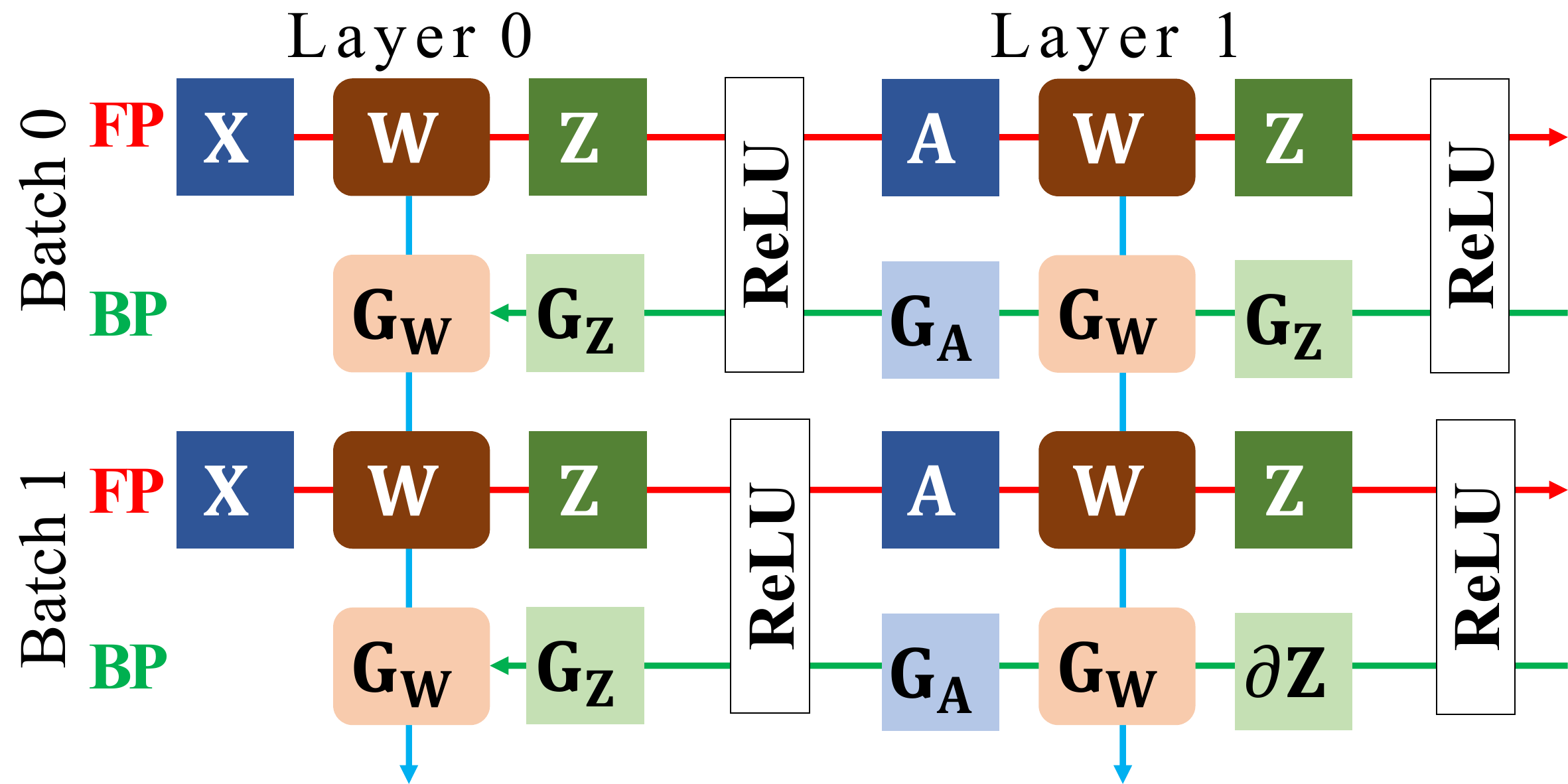}
        \caption{Prover executes training: Forward Propagation (FP, red arrows) \& Backward Propagation (BP, green arrows), and parameter updates (blue arrows).} %
        \label{fig:overview-train}
    \end{subfigure}
    \hfill
    \begin{subfigure}[t]{0.3\textwidth}
        \centering
        \includegraphics[page=4,width=\textwidth]{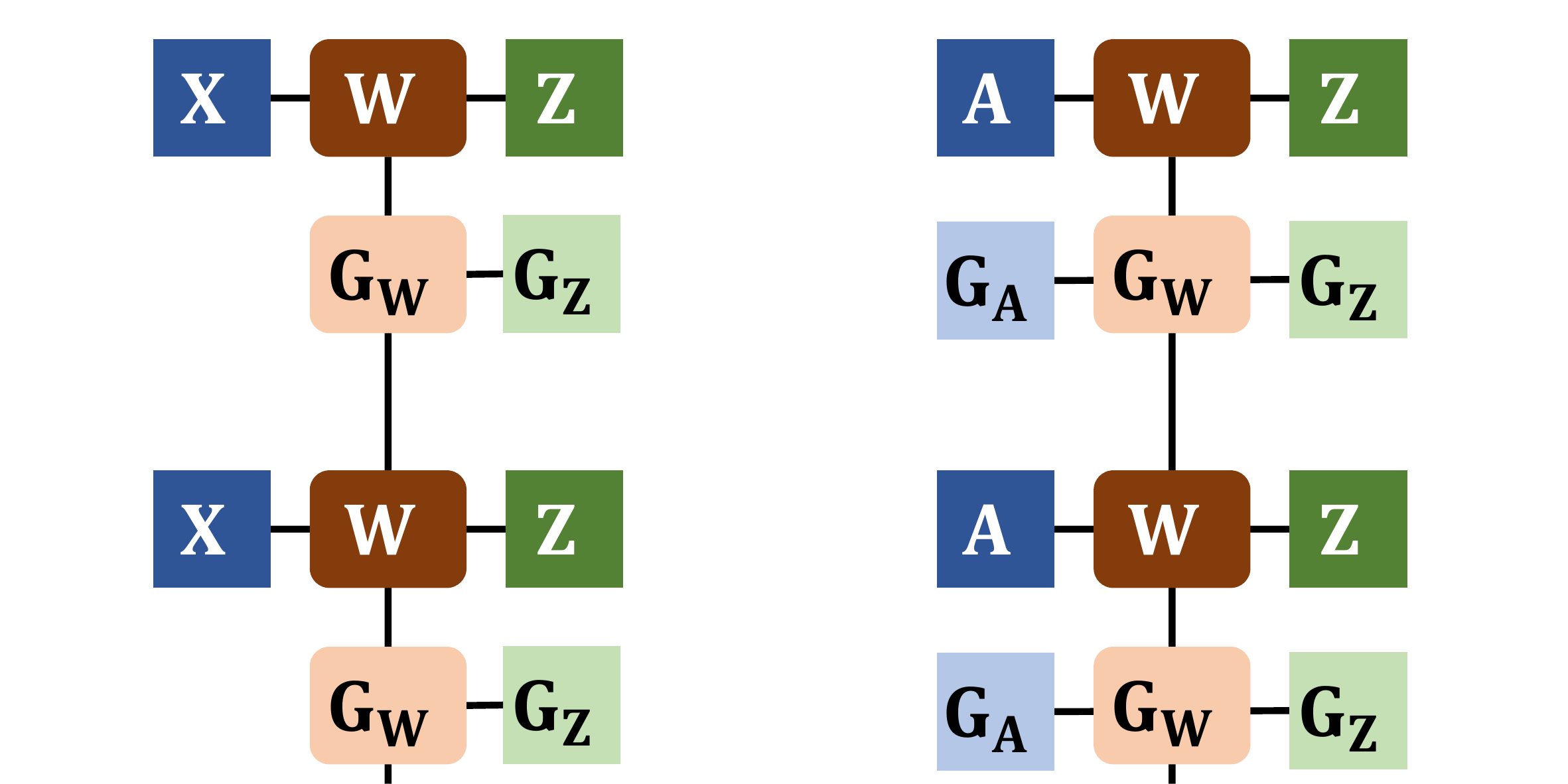}
        \caption{Activation functions disrupt arithmetic relations (black lines), compromising proof integrity.}
        \label{fig:overview-disconnect}
    \end{subfigure}
    \hfill
    \begin{subfigure}[t]{0.3\textwidth}
        \centering
        \includegraphics[page=5,width=\textwidth]{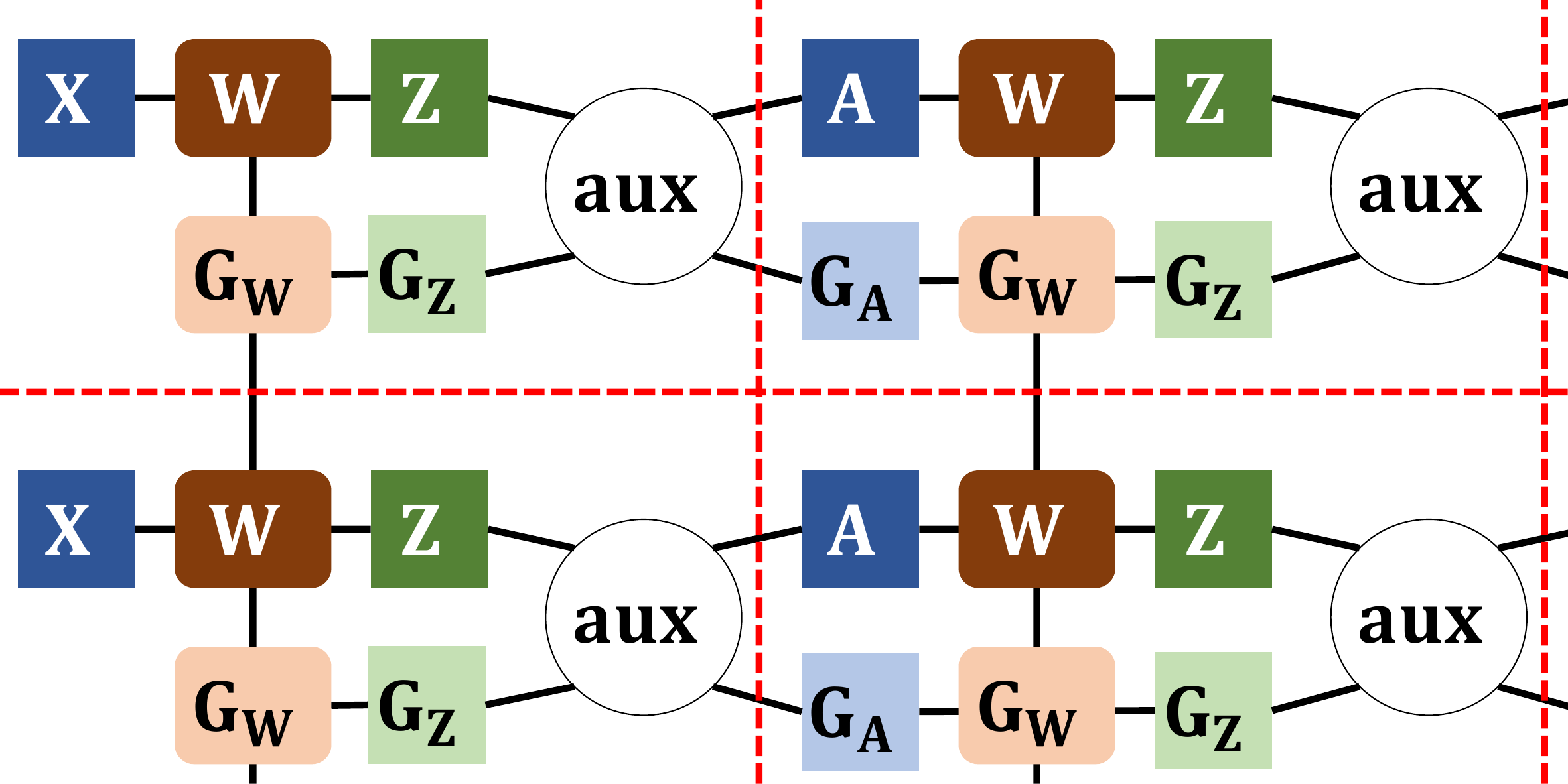}
        \caption{Prover uses auxiliary inputs (absolute values and signs) to reconnect tensors, restoring arithmetic relations. Note that the same pattern is repeated in each cell.}
        \label{fig:overview-reconnect}
    \end{subfigure}
    
    \begin{subfigure}[t]{0.3\textwidth}
        \centering
        \includegraphics[page=10,width=\textwidth]{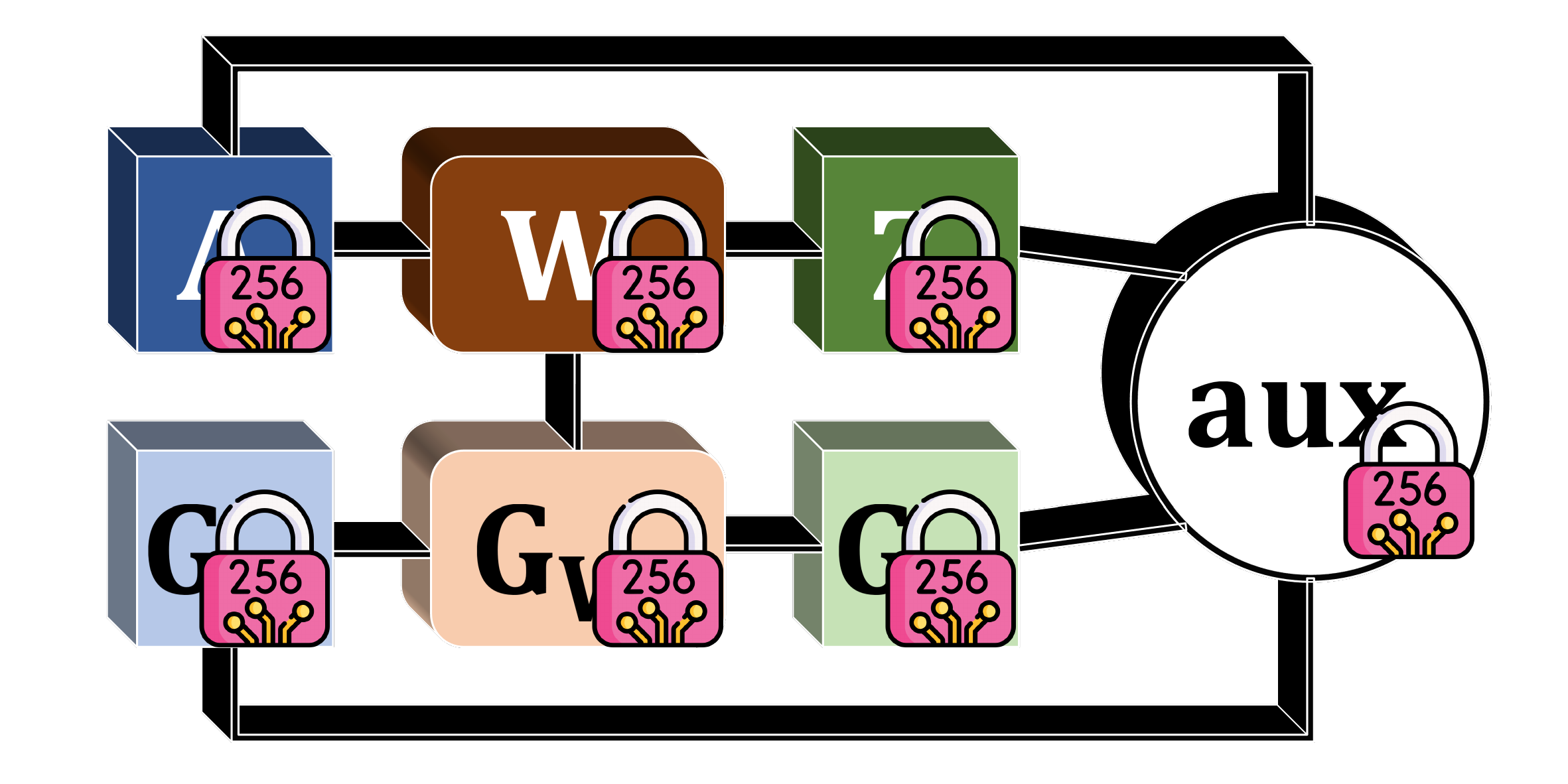}
        \caption{Due to the repeated pattern, the prover stacks tensors across layers and steps, commits to them (represented by the lock), and batches the arithmetic relations accordingly.} 
        \label{fig:overview-stack}
    \end{subfigure}
    \hfill
    \begin{subfigure}[t]{0.3\textwidth}
        \centering
        \includegraphics[page=11,width=\textwidth]{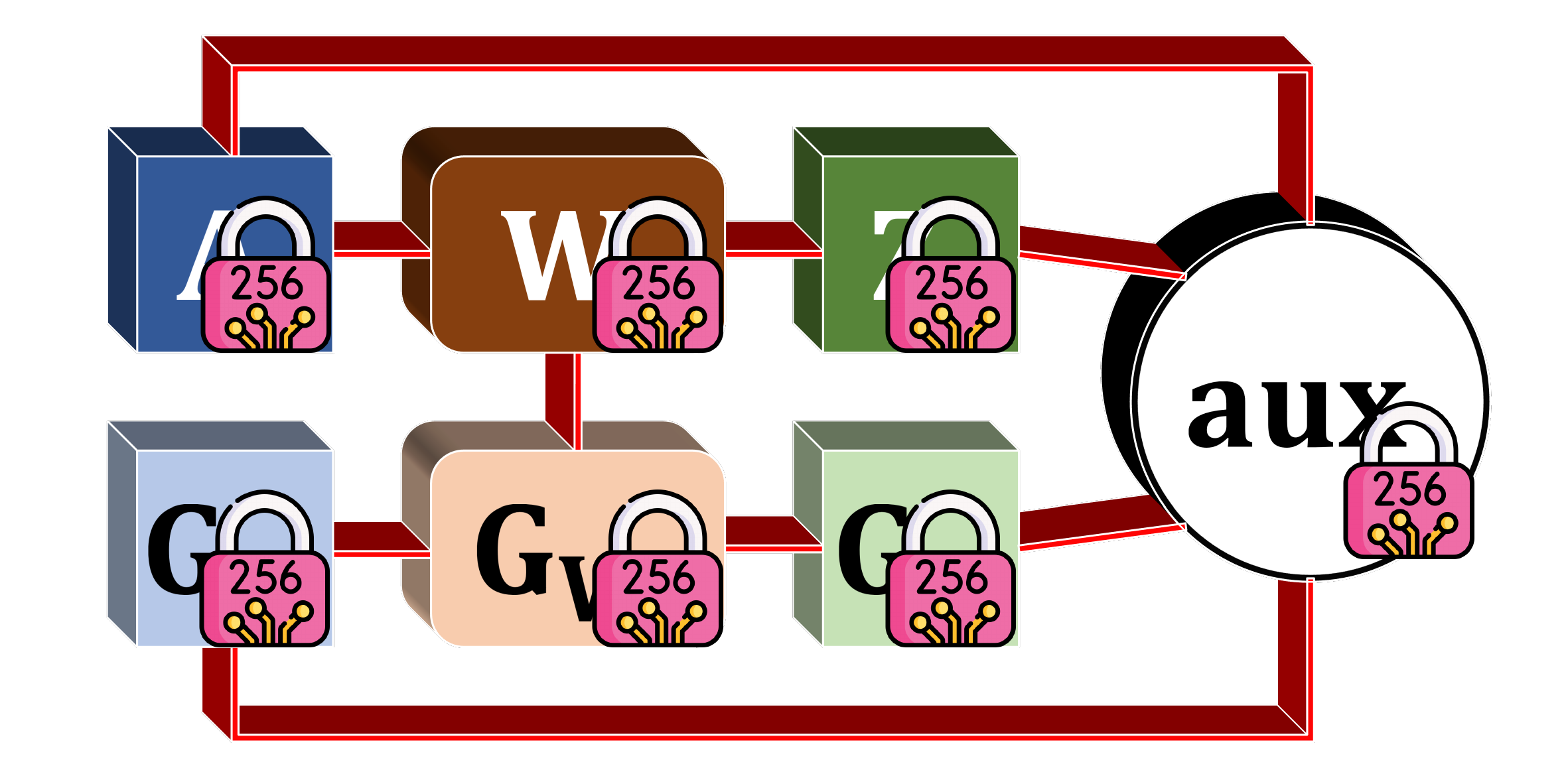}
        \caption{Generating batched proof on arithmetic relations (highlighted in red) among stacked tensors.}
        \label{fig:overview-prove}
    \end{subfigure}
    \hfill
    \begin{subfigure}[t]{0.3\textwidth}
        \centering
        \includegraphics[page=12,width=\textwidth]{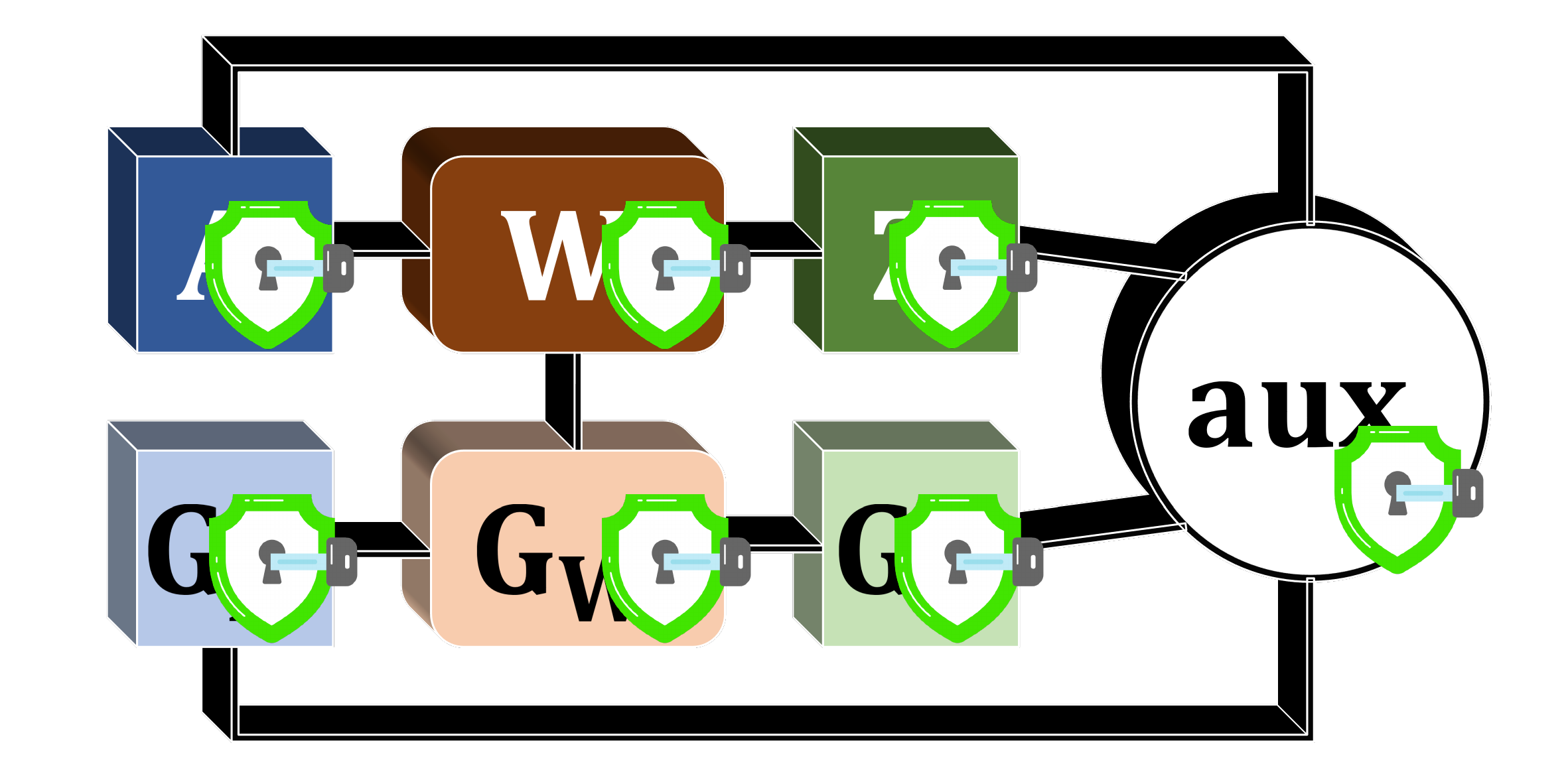}
        \caption{Concluding with proof of evaluations (represented by unlocking) on stacked tensors.}
        \label{fig:overview-open}
    \end{subfigure}
    
    \caption{Overview of zkDL. One of our key contributions is the \textbf{compressed proof} across data batches, DL layers and training steps, by leveraging the repeated structures in deep learning training for acceleration.}
    \label{fig:overview}
\end{figure*}

In this section, we present an overview of zkDL, as depicted in Figure \ref{fig:overview}. zkDL, tailored as a Zero-Knowledge Proof (ZKP) for deep learning training, models both forward and backward propagations within neural networks (NNs) as arithmetic circuits (ACs). It adeptly manages non-arithmetic operations, which inherently resist direct proof. Furthermore, leveraging the inherent relatedness of tensor-based and layered structures across all training steps, zkDL effectively batches and compresses the proof. This strategic approach significantly diminishes the computational and communicational burdens for both the prover and verifier.

\subsubsection{Threat model} We assume the presence of two probabilistic polynomial-time (PPT) parties: a prover \(\mathcal{P}\) (e.g., an AI company) and a verifier \(\mathcal{V}\) (e.g., a government agency). Both parties concur on a predetermined neural network structure and training logic, including aspects such as the number of training steps, optimization algorithms, and learning rates. The prover's role is to strictly adhere to the training logic when training the model and then to demonstrate this adherence to the verifier. The verifier, while being semi-honest in adhering to the prescribed protocol for validation, remains interested in the training data and model parameters privately held by the prover.

\subsubsection{Haunt of non-arithmetic and zkReLU} \label{sec:overview-ac} As illustrated in Figure \ref{fig:overview-train}, training NNs primarily comprises two components: tensor operations within each layer and the forward and backward propagation through activation functions, exemplified by ReLU, situated between the layers. In the realm of ZKP, the intra-layer operations such as matrix multiplications, convolutions, and their backward propagations correspond to arithmetic operations. These operations are inherently composed of additions, subtractions, and multiplications, aligning naturally with ZKP schemes. In contrast, ReLU, akin to many other activation functions, is inherently non-arithmetic. This distinction necessitates dedicated ZKP protocols for its forward and backward propagations, ensuring they integrate coherently with the proof mechanisms for the training's arithmetic-centric segments.

As depicted in Figure \ref{fig:overview-disconnect}, due to the nature of ReLU and its backward propagation not establishing input-output relationships based on any arithmetic operation, the arithmetic connections between successive NN layers are absent. This absence necessitates that the dangling inputs and outputs related to ReLU's forward and backward passes—specifically, the preactivation $\mathbf{Z}$, the activation $\mathbf{A}$, and their respective gradients $\mathbf{G}_\mathbf{Z}$ and $\mathbf{G}_\mathbf{A}$, each of the same dimension $D$—be bound by commitments. They cannot merely be considered as intermediate computational values over AC. Additionally, it is imperative to reinstate this connection to hinder any potential deceitful actions by the prover during the ReLU's forward and backward computation, all while ensuring that the arithmetic operations within the layers are computed correctly and pass the verification process.

To reestablish this connection, the terms $\mathbf{Z}, \mathbf{G}_\mathbf{A}$ and $\mathbf{A}, \mathbf{G}_\mathbf{Z}$ can no longer be viewed as the ``input'' and ``output'' for forward and backward propagations. We decompose these values into an auxiliary input, represented as $\mathbf{aux}=(\mathbf{Z}', \mathbf{S_Z})$. Here, $\mathbf{Z}'$ denotes the absolute value of $\mathbf{Z}$, and $\mathbf{S_Z}$ represents the sign of $\mathbf{Z}$. The design of zkReLU should ensure that these auxiliary inputs can effectively reconstruct the four tensors, which are separated due to non-linearity. Given a randomness $r$ chosen by the verifier, Schwartz–Zippel lemma \cite{DBLP:journals/jacm/Schwartz80, DBLP:conf/eurosam/Zippel79} guarantees that the verification can be represented by: \begin{equation}
    \mathbf{Z} + \mathbf{Z}' + r\mathbf{A} + r^2\mathbf{G}_\mathbf{Z} = \left((2+r)\mathbf{Z}' +r^2 \mathbf{G}_\mathbf{A} \right)\odot \mathbf{S}_\mathbf{Z},
\end{equation} or equivalently, \begin{equation}
    \left(\mathbf{Z} - \mathbf{Z}' \odot (2\mathbf{S}_\mathbf{Z} - 1)\right) + \left(\mathbf{A} - \mathbf{Z}' \odot \mathbf{S}_\mathbf{Z}\right)r + \left(\mathbf{G}_\mathbf{Z} - \mathbf{G}_\mathbf{A} \odot \mathbf{S}_\mathbf{Z}\right)r^2 = 0. \label{eq:overview-zkrelu}
\end{equation} This ensures that the coefficient of each term is zero, and therefore the correctness of the ReLU operation.

However, while the incorporation of $\mathbf{aux}$ addresses the gap resulting from the ReLU non-linearity, the original logical orderings between the layers are not retained. Yet, this seemingly disadvantageous situation opens the door for a surprising optimization. Specifically, the tensor operations, including ReLU augmented with $\mathbf{aux}$, from various layers and training steps can be aggregated regardless of their logical orderings in training, which has already been disrupted by zkReLU.

\subsubsection{FAC4DNN: Aggregating and compressing proofs of deep learning training}

In Figure \ref{fig:overview-reconnect}, a consistent pattern of tensors—including auxiliary inputs—and their corresponding arithmetic relations is evident, spanning multiple layers and training steps. During the training phases, the prover must adhere to the designated sequence of layers and training steps to finalize the computations intrinsic to the training process. However, as the disconnection of the arithmetic circuit and the commitment to the auxiliary inputs fundamentally rewire the arithmetic circuit, tensors in \textbf{different NN layers and training steps} reside in the \textbf{same AC layer} for the proof generation. Therefore, the prover can collectively group the proofs across these repetitive patterns.

As illustrated in Figure \ref{fig:overview-stack}, throughout the proof generation phase, the prover compiles similar tensors from all repetitive units and simultaneously amalgamates their arithmetic relationships. To delve deeper, consider \(N\) instances of matrix multiplication, a common tensor operation in deep learning, such as \(\mathbf{Z}^{(n)}\gets \mathbf{X}^{(n)} \mathbf{Y}^{(n)}\,(0\leq n \leq N-1)\). These instances span several layers and training steps that require verification. Each matrix is defined as \(\mathbf{Z}^{(n)}\in \F^{D_1\times D_3}, \mathbf{X}^{(n)}\in \F^{D_1\times D_2}\), and \(\mathbf{Y}^{(n)}\in \F^{D_2\times D_3}\). Instead of verifying all \(N\) instances individually, the proof directly addresses the stacked matrices \(\mathbf{X}\in \F^{N \times D_1\times D_2}, \mathbf{Y}\in \F^{N \times D_2\times D_3}, \mathbf{Z}\in \F^{N \times D_1\times D_3}\), which are in the form of 3-order tensors. Using randomness designated by the verifier, namely \(\mathbf{w}\sim \F^{\log_2 N}, \mathbf{u}_1\sim \F^{\log_2 D_1}, \mathbf{u}_3\sim \F^{\log_2 D_3}\), the aggregated proof is represented as: 
\begin{equation} \label{eq:overview-sumcheck-aggr}
    \me{\mathbf{Z}}{\mathbf{w}, \mathbf{u}_1, \mathbf{u}_3} = \sum_{\mathbf{i}=0}^{N}\me{\beta}{\mathbf{w}, \mathbf{i}}\sum_{\mathbf{j}=0}^{D_2} \me{\mathbf{X}}{\mathbf{i}, \mathbf{u}_1, \mathbf{j}} \me{\mathbf{Y}}{\mathbf{i}, \mathbf{j}, \mathbf{u}_3},
\end{equation}
where \(\me{\mathbf{Z}}{\cdot}, \me{\mathbf{X}}{\cdot}, \me{\mathbf{Y}}{\cdot}, \me{\beta}{\cdot}\) are multilinear extensions \cite{me} of \(\mathbf{Z}, \mathbf{X}, \mathbf{Y}\) (viewed as a function mapping from binary representations of indices to values of the corresponding dimensions) and \(\beta: \{0, 1\}^{\log_2N}\times \{0, 1\}^{\log_2N} \to \{0, 1\}\) where \(\beta(\mathbf{b}_1, \mathbf{b}_2) = \1\left\{\mathbf{b}_1 = \mathbf{b}_2\right\}\). 

Notably, executing the sumcheck \eqref{eq:overview-sumcheck-aggr} over $\mathbf{i}$, which pertains to the extra dimension from stacking, compresses the proof of \eqref{eq:overview-sumcheck-aggr} to: \begin{equation}
    \alpha = \me{\beta}{\mathbf{w}, \mathbf{v}}\sum_{\mathbf{j}=0}^{D_2} \me{\mathbf{X}}{\mathbf{v}, \mathbf{u}_1, \mathbf{j}} \me{\mathbf{Y}}{\mathbf{v}, \mathbf{j}, \mathbf{u}_3}.
\end{equation} In this context, the prover's claim $\alpha$ and randomness $\mathbf{v}$ arise from the compression's execution. Therefore, validating: \begin{equation}
    \alpha {\me{\beta}{\mathbf{w}, \mathbf{v}}}^{-1}= \sum_{\mathbf{j}=0}^{D_2} \me{\mathbf{X}}{\mathbf{v}, \mathbf{u}_1, \mathbf{j}} \me{\mathbf{Y}}{\mathbf{v}, \mathbf{j}, \mathbf{u}_3}
\end{equation} revives the sumcheck for singular matrix multiplication.

This technique of aggregation and compression extends to all tensor operations that can be verified using the sumcheck protocol individually, regardless of the number of inputs. Specifically, all $N$ instances of each input are condensed into one, contingent upon the randomness of $\mathbf{v}$. This ensures a seamless transition to the sumcheck for singular operations. As a result, the original logical ordering of the $N$ operational instances, whether they belong to consecutive layers or training phases, becomes inconsequential. From an arithmetic circuit (AC) perspective, this method effectively ``flattens'' the circuit, reducing its depth by a magnitude of $O(N)$. However, this comes at the cost of expanding the width of each AC layer by the same magnitude.

Building on zkReLU, it is inescapable that tensor commitments occur at every layer and training step, exemplified by each repetitive unit in Figure \ref{fig:overview-reconnect}. Hence, there is no asymptotic overhead in the cumulative tensor size bound by the commitments, nor in the requisite commitment time. Conversely, assuming the adoption of sub-linear size commitments, such as Hyrax \cite{hyrax} in this study (whose commitment sizes grow as fast as the square root of the overall committed value size), committing to a stack of $N$ tensors could culminate in an asymptotic reduction in the overall commitment size by an order of $O\left(\sqrt{N}\right)$. 

Similarly, due to the flattened design of the arithmetic circuit facilitated by FAC4DNN, there are significant reductions in both proof sizes and verification times. By aggregating $N$ instances of the same operation, the proof size does not increase at a $\Theta\left(N\right)$ rate. Only a single copy of the proof for a singular tensor operation is needed, instead of $N$ copies, and the compression steps introduce only an $O(\log N)$ overhead. Additionally, when summing up the components of zkDL, the proof verification depicted in Figure \ref{fig:overview-open} also witnesses a reduction by an order of $O\left(\sqrt{N}\right)$. Within zkDL, the batched tensors seamlessly integrate into the parallel computing environment of deep learning training, further reducing the total proof time, especially since there is no need to strictly follow the original computation sequence throughout training.

\section{Preliminaries}

\subsection{Notations}

In this study, vectors and tensors are denoted using boldface notation, such as \( \mathbf{v} \) and \( \mathbf{S} \). To ensure compatibility with the cryptographic tools employed, we adopt a 0-indexing scheme for these vectors and tensors. Our indexing approach for multi-dimensional tensors aligns with the PyTorch convention, exemplified by \( \mathbf{S}_{[i, j_0:j_1, :]} \). Additionally, for any positive integer \( N \), we employ the shorthand notation \( [N] \) to represent the set \( \{0, 1, \dots, N-1\} \).

\subsection{Sumcheck and GKR protocols} \label{sec:sumcheck}

The sumcheck protocol \cite{sumcheck, me} serves as a fundamental component in modern proof systems, allowing for the verification of the correctness of the summation $\sum_{\mathbf{b} \in \{0, 1\}^d} f(\mathbf{b})$ for a $d$-variate polynomial $f$. This protocol offers an efficient proving time of $O(2^d)$ and a compact proof size of $O(d)$.

Building upon the sumcheck protocol, the GKR protocol \cite{gkr} provides an interactive proof for the correct computation of arithmetic circuits. It leverages the sumcheck protocol between the layers of the arithmetic circuit, as well as the Pedersen commitments to the private inputs. 

The sumcheck and GKR protocols have found wide applications in verifying the proper execution of deep learning models, thanks to their compatibility with tensor structures. In particular, the tensor operations can often be expressed in the form of sumchecks, via the \emph{multilinear extensions} of the tensors: for each tensor $\mathbf{S} \in \F^D$ that is discretized from real numbers (without loss of generality, assume $D$ is a power of 2, or zero-padding may be applied), its multilinear extension $\me{\mathbf{S}}{\cdot}: \F^{\log_2D}\to \F$ is a multivariate polynomial defined as
\begin{equation}
    \label{eq:multilinear-extension}
    \me{\mathbf{S}}{\mathbf{u}} = \sum_{\mathbf{b}\in \{0, 1\}^{\log_2D}}\mathbf{S}(\mathbf{b})\me{\beta}{\mathbf{u}, \mathbf{b}},
\end{equation}
where $\mathbf{b}$ represents the $\mathbf{b}$-th element of $\mathbf{S}$ (identifying the index by the binary string), and $\me{\beta}{\cdot, \cdot}: \F^{\log_2D}\times \F^{\log_2D}\to \F$ is a polynomial. When restricted to ${\{0, 1\}}^{\log_2D}\times {\{0, 1\}}^{\log_2D}$, $\me{\beta}{\mathbf{b}_1, \mathbf{b}_2} = \case{1, & \text{if }\mathbf{b}_1 = \mathbf{b}_2; \\ 0, & \text{if }\mathbf{b}_1 \neq \mathbf{b}_2,}$ for $\mathbf{b}_1, \mathbf{b}_2 \in {\{0, 1\}}^{\log_2D}$. In the context of multilinear extensions, we use the notation of indices and their binary representations interchangeably. Moreover, we use $\me{S}{\mathbf{u}_1, \mathbf{u}_2, \dots, \mathbf{u}_k}$ to denote the evaluation of $\me{S}{\cdot}$ at the concatenation of multiple vectors $\mathbf{u}_1, \mathbf{u}_2, \dots, \mathbf{u}_k$ whose dimensions sum up to $\log_2D$. 

Specialized adaptations of the sumcheck protocols cater to prevalent deep learning operations like matrix multiplication \cite{DBLP:conf/crypto/Thaler13, safetynets} and convolution \cite{zkcnn}. When the sumcheck protocol is applied to these operations, it yields the proclaimed evaluation \(\me{\mathbf{S}}{\mathbf{v}}\) of the multilinear extension for each involved tensor \(\mathbf{S}\), with \(\mathbf{v}\) being selected due to the inherent randomness of the protocol. Subsequent verification of these claimed evaluations employs the proof of evaluations relative to the commitment of \(\mathbf{S}\), as detailed in \ref{sec:pedersen}. Additionally, zero-knowledge variants of the sumcheck protocol \cite{DBLP:journals/eccc/ChiesaFS17, libra, orion} have been developed with asymptotically negligible overhead, which leaks no information on the tensors involved once employed.

Historically, representing NNs as ACs has posed challenges due to the inclusion of non-arithmetic operations, making it ambiguous and inefficient. In this study, we tackle this problem by refining and enhancing the modelling approach. Rather than directly representing NNs as ACs with identical layer structures and invoking the GKR protocol, we adeptly aggregate the sumcheck protocols for tensor operations over different layers and training steps, grounded in our novel AC design of FAC4DNN. This leads to improved running times and more concise proof sizes.

\subsection{Pedersen commitments} \label{sec:pedersen}

The Pedersen commitment is a zero-knowledge commitment scheme that relies on the hardness of the discrete log problem (DLP). Specifically, in a finite field $\F$ with prime order $p$, committing to $d$-dimensional vectors requires an order-$p$ cyclic group $\G$ (e.g., an elliptic curve) and uniformly independently sampled values $\mathbf{g} = (g_0, g_1, \dots, g_{d-1})^\top \sim \G^d$, and $h \sim \G$. This scheme allows any $d$-dimensional tensor $\mathbf{S} = (S_0, S_1, \dots, S_{d-1})^\top \in \F^d$ to be committed as: \[\com{\mathbf{S}}\gets \tt{Commit}(\mathbf{S}; r) = h^r\mathbf{g}^\mathbf{S} = h^r\prod_{i=0}^{d-1}g_i^{S_i},\] where $r\sim \F$ is uniformly sampled, ensuring zero-knowledgeness of the committed value $\mathbf{S}$.  A complete and sound \emph{proof of evaluation} can be conducted for $\me{\mathbf{S}}{\mathbf{v}}$ with respect to $\com{\mathbf{S}}$ for any randomness $\mathbf{v}$ chosen by the verifier. This can be utilized as a component of the proofs for operations on private tensors.

Additionally, the Pedersen commitment scheme exhibits homomorphic properties. Specifically, for two commitments, $\com{\mathbf{S}_1} = h^{r_1}\mathbf{g}^{\mathbf{S}_1}$ and $\com{\mathbf{S}_2} = h^{r_2}\mathbf{g}^{\mathbf{S}_2}$, corresponding to tensors $\mathbf{S}_1$ and $\mathbf{S}_2$, their multiplied result, $\com{\mathbf{S}_1}\cdot \com{\mathbf{S}_2} = h^{r_1 + r_2}\mathbf{g}^{\mathbf{S}_1 + \mathbf{S}_2}$, is a valid commitment to the sum, $\mathbf{S}_1 + \mathbf{S}_2$.

Incorporating the Pedersen commitment leads to a \(O(d)\) runtime, both for committing to a tensor and for conducting a proof of evaluation from the prover's side. In real-world applications, several variations of the Pedersen commitment are utilized to enhance verifier efficiency and curtail communication demands. For example, Hyrax \cite{hyrax} is a commitment scheme that does not require a trusted setup, refining the commitment size, proof of evaluation size, and the time it takes for a verifier to evaluate the proof to $O(\sqrt{d}), O(\log d)$, and $O(\sqrt{d})$, respectively. These advancements are strategically integrated into the blueprint of zkDL, particularly FAC4DNN, aiming to improve the prover time, proof sizes, and verifier times in the realm of deep learning.

\subsection{Security assumptions}
We assume that the commitment scheme employed in our research offers \(\lambda\)-bit security. In line with this, the finite field \(\F\), wherein all computations are discretized, is of size \(\Omega(2^{2\lambda})\). Furthermore, we posit that all aspects of the training procedure, encompassing the number of training steps, the number of layers, tensor dimensions, and the complexity of operations between them, are all polynomial in \(\lambda\).

\section{zkReLU: Proof of forward and backward propagations through ReLU activation}\label{sec:zkrelu}

The proper and tailored handling of non-linearities, especially ReLU, is essential to achieve efficient zero-knowledge verifiable training on deep neural networks. In this section, we introduce zkReLU, a zero-knowledge protocol designed specifically to verify the training of deep neural networks that incorporate ReLU non-linearity. Our scheme employs auxiliary inputs, allowing for the verification of both the forward and backward propagations involving ReLU. Furthermore, zkReLU integrates the ReLU function into the FAC4DNN framework, which is primarily concerned with the arithmetic operations between tensors. This integration ensures that the efficiencies brought about by FAC4DNN extend to the proof of the entire training process.

When the ReLU activation function is applied to the output of layer \( \ell \) (with \( 1 \leq \ell \leq L-1 \) and \( L \) representing the total number of layers), denoted as \( \mathbf{Z}^{(\ell)} \), it pertains to a linear layer, either fully connected or convolutional. Given that multiplication operations play a role in computing \( \mathbf{Z}^{(\ell)} \), \( \mathbf{Z}^{(\ell)} \) undergoes scaling twice by the scaling factor, assumed to be a power of 2, specifically \( 2^R \). Consequently, when dealing with quantized values, the ReLU operation must also reduce the input by a factor of \( 2^R \). This mechanism can be articulated as the activation function \( \mathbf{A}^{(\ell)} = \text{ReLU}\left(\left\lfloor\frac{\mathbf{Z}^{(\ell)}}{2^R}\right\rceil\right) = \1\left\{\left\lfloor\frac{\mathbf{Z}^{(\ell)}}{2^R}\right\rceil \geq 0\right\}\odot \left\lfloor\frac{\mathbf{Z}^{(\ell)}}{2^R}\right\rceil \).

To simplify the notation, we introduce the rescaled \({\mathbf{Z}^{(\ell)}}' := \left\lfloor\frac{\mathbf{Z}^{(\ell)}}{2^R}\right\rceil\). This representation allows for the expression of \(\mathbf{Z}^{(\ell)}\) as \(\mathbf{Z}^{(\ell)} = 2^R {\mathbf{Z}^{(\ell)}}' + \mathbf{R}_\mathbf{Z}^{(\ell)}\), where \(\mathbf{R}_\mathbf{Z}^{(\ell)}\) denotes the remainder resulting from rounding. To adequately define the concept of "non-negative" within the finite field, it becomes necessary to restrict the scale of \({\mathbf{Z}^{(\ell)}}'\). We assume each element of \({\mathbf{Z}^{(\ell)}}'\) is an \(Q\)-bit signed integer, with \(2^Q \ll |\mathbb{F}|\). Solely for analytical reasons, we decompose \({\mathbf{Z}^{(\ell)}}'\) into its magnitude bits and sign bits, such that \({\mathbf{Z}^{(\ell)}}' = \sum_{j=0}^{Q-2}2^j\mathbf{B}_j^{(\ell)} - 2^{Q-1}\mathbf{B}_{Q-1}^{(\ell)}\), with each \(\mathbf{B}_j^{(\ell)}\) for \(0\leq j \leq Q-1\) being binary. Furthermore, \(\mathbf{B}_{Q-1}^{(\ell)}\) represents the negativity of each dimension in \(\mathbf{Z}^{(\ell)}\) (assigning 1 for negative values and 0 otherwise). The arithmetic relations between these intermediate values and the input and outputs of the forward propagation, notably \(\mathbf{A}^{(\ell)}\) and \(\mathbf{Z}^{(\ell)}\), can be captured as: \begin{align}
    \label{eq:zkrelu-A} \mathbf{A}^{(\ell)} &= (\mathbf{1}-\mathbf{B}_{Q-1}^{(\ell)}) \odot {\mathbf{Z}^{(\ell)}}',\\
    \label{eq:zkrelu-Z} \mathbf{Z}^{(\ell)} &= 2^R {\mathbf{Z}^{(\ell)}}' + {\mathbf{R}_\mathbf{Z}^{(\ell)}}.
\end{align}

During the backpropagation phase, the gradient of \(\mathbf{A}^{(\ell)}\), denoted as \(\mathbf{G}_\mathbf{A}^{(\ell)}\), is typically scaled twice by \(2^R\) owing to the multiplication operations involved. As such, it becomes necessary for the prover to rescale this gradient to \({\mathbf{G}_\mathbf{A}^{(\ell)}}' := \left\lfloor\frac{\mathbf{G}_\mathbf{A}^{(\ell)}}{2^R}\right\rceil\), with the resulting remainder being \(\mathbf{R}_{\mathbf{G}_\mathbf{A}}^{(\ell)}\). Subsequently, \({\mathbf{G}_\mathbf{A}^{(\ell)}}'\) is employed to compute the gradient of \(\mathbf{Z}^{(\ell)}\), represented as \(\mathbf{G}_\mathbf{Z}^{(\ell)}\), through the Hadamard product \(\odot\) with \(\mathbf{1}-\mathbf{B}_{Q-1}^{(\ell)}\). Analogous to the forward propagation, the correctness of the backward propagation can be outlined through the following arithmetic relations: \begin{align}
    \label{eq:zkrelu-GZ} \mathbf{G}_\mathbf{Z}^{(\ell)} &= (\mathbf{1}-\mathbf{B}_{Q-1}^{(\ell)}) \odot {\mathbf{G}_\mathbf{A}^{(\ell)}}',\\
    \label{eq:zkrelu-GA} \mathbf{G}_\mathbf{A}^{(\ell)} &= 2^R {\mathbf{G}_\mathbf{A}^{(\ell)}}' + \mathbf{R}_\mathbf{\mathbf{G}_\mathbf{A}}^{(\ell)}.
\end{align}

It is also important to observe that for the intermediate variables involved in \eqref{eq:zkrelu-A}, \eqref{eq:zkrelu-Z}, \eqref{eq:zkrelu-GZ} and \eqref{eq:zkrelu-GA}, the values of the tensors each tensor need to be constrained so as to prevent malicious manipulations by the prover. Namely, ${\mathbf{Z}^{(\ell)}}'\in \left[-2^{Q-1}, 2^{Q-1}\right)^{D^{(\ell)}}$, $\mathbf{B}_{Q-1}^{(\ell)}\in \{0, 1\}^{D^{(\ell)}}$, $\mathbf{R}_\mathbf{Z}^{(\ell)}\in \left[-2^{R-1}, 2^{R-1}\right)^{D^{(\ell)}}$, ${\mathbf{G}_\mathbf{A}^{(\ell)}}' \in \left[-2^{Q-1}, 2^{Q-1}\right)^{D^{(\ell)}}$, and $\mathbf{R}_{\mathbf{G}_\mathbf{A}}^{(\ell)}\in \left[-2^{R-1}, 2^{R-1}\right)^{D^{(\ell)}}$, are bounded and share the same dimension $D^{(\ell)}$. However, in compatibility with our design of FAC4DNN, which is overviewed in Section \ref{sec:overview-ac} and detailed in \ref{sec:ac}, the proof is aggregated over multiple training steps and layers. Therefore, in the following discussion, we use the notations without the superscripts (i.e., $\mathbf{A}, \mathbf{Z}, \mathbf{G_Z}, \mathbf{G_A}$ for the input and outputs of both directions of ReLU, and $\mathbf{Z}', \mathbf{B}_{Q-1}, \mathbf{R_Z}, \mathbf{G}_\mathbf{A}', \mathbf{R_{G_A}}$ as the intermediate values introduced) to represent the stacked tensors over multiple layers, all of which share the same dimension denoted as $D$. 

\subsection{Formulation of zkDL} \label{sec:zkdl-formulation}

As an initial step, we presuppose that the proof's execution over other components yields the claimed evaluations of the multilinear extensions on \(\mathbf{A}, \mathbf{Z}, \mathbf{G_Z}, \mathbf{G_A}\). This is achieved through the aggregated sumchecks on the other operations in which these tensors participate, as will be elaborated in Section \ref{sec:ac}. These are represented as \(\me{\mathbf{A}}{\mathbf{u_A}}, \me{\mathbf{Z}}{\mathbf{u_Z}}, \me{\mathbf{G_Z}}{\mathbf{u_{G_Z}}}, \me{\mathbf{G_A}}{\mathbf{u_{G_A}}}\). Thus, by committing to the intermediate values \(\mathbf{Z}', \mathbf{B}_{Q-1}, \mathbf{R_Z}, \mathbf{G}_\mathbf{A}', \mathbf{R_{G_A}}\) and executing the sumcheck protocol on the aggregated versions of  \eqref{eq:zkrelu-A}, \eqref{eq:zkrelu-Z}, \eqref{eq:zkrelu-GZ}, \eqref{eq:zkrelu-GA}, the validity of these four equations can be verified with overwhelming probability.

Nevertheless, the validity criteria for the intermediate values also warrant consideration. To ensure complete compatibility with FAC4DNN, which functions over aggregated tensor structures, these intermediate values are collectively represented as a 3D binary tensor— the auxiliary input \(\mathbf{aux} \in \{0, 1\}^{2\times D\times (Q + R)}\). Here, \(\mathbf{aux}_{[0, :, :]}\) and \(\mathbf{aux}_{[1, :, :]}\) stand for binary representations of the \((Q+R)\)-bit integers in \(\mathbf{Z}\) and \(\mathbf{G_A}\) respectively, given by: \begin{align}
    \mathbf{aux}_{[0, :, :]}\mathbf{s}_{Q+R} &= \mathbf{Z}\label{eq:aux-Z},\\
    \mathbf{aux}_{[1, :, :]}\mathbf{s}_{Q+R} &= \mathbf{G_A}.\label{eq:aux-GA}
\end{align}
Here, \(\mathbf{s}_B = (1, 2, 2^2, \dots, 2^{B-2}, -2^{B-1})^\top\) facilitates the recovery of a \(B\)-bit integer from its binary representations.

Using this configuration, the intermediate variables can be equated as: \begin{align}
    \mathbf{aux}_{[0, :, R:Q+R]}\mathbf{s}_{Q} + \mathbf{aux}_{[0, :, R-1]} &= \mathbf{Z}',\\
    \mathbf{aux}_{[1, :, R:Q+R]}\mathbf{s}_{Q} + \mathbf{aux}_{[1, :, R-1]} &= \mathbf{G}_\mathbf{A}',\\
    \mathbf{aux}_{[0, :, 0:R]}\mathbf{s}_{R} &= \mathbf{R_Z},\\
    \mathbf{aux}_{[1, :, 0:R]}\mathbf{s}_{R} &= \mathbf{R_{G_A}},\\
    \mathbf{aux}_{[0, :, Q+R-1]} &= \mathbf{B}_{Q-1}.
\end{align} Ensuring \eqref{eq:zkrelu-Z} and \eqref{eq:zkrelu-GA} is upheld, while \eqref{eq:zkrelu-A} and \eqref{eq:zkrelu-GZ} can be reframed as: \begin{gather}
    \mathbf{A} = (\mathbf{1} - \mathbf{aux}_{[0, :, Q+R-1]}) \odot \left(\mathbf{aux}_{[0, :, R:Q+R]}\mathbf{s}_{Q} + \mathbf{aux}_{[0, :, R-1]}\right), \label{eq:aux-A}\\
    \mathbf{G_Z} = (\mathbf{1} - \mathbf{aux}_{[0, :, Q+R-1]}) \odot \left(\mathbf{aux}_{[1, :, R:Q+R]}\mathbf{s}_{Q} + \mathbf{aux}_{[1, :, R-1]}\right). \label{eq:aux-GZ}
\end{gather} Furthermore, to ensure that \(\mathbf{aux}\) is genuinely binary, the \emph{auxiliary input validity proof (AIVP)} must be conducted on \begin{equation}
    \mathbf{aux} \odot \left(\mathbf{aux} - \mathbf{1}\right) = \mathbf{0}. \label{eq:aux-bin}
\end{equation} Implementing the sumcheck protocol on \eqref{eq:aux-Z}, \eqref{eq:aux-GA}, \eqref{eq:aux-A}, \eqref{eq:aux-GZ}, and \eqref{eq:aux-bin} is sufficient to ascertain the correctness of the computation for ReLU with respect to \(\mathbf{Z}, \mathbf{G_A}, \mathbf{A}\), and \(\mathbf{G_Z}\). Practically, zkReLU combines the sumcheck protocols for these equations to further diminish the running times and size of the proof, chiefly by confining the relatively costly proof of evaluation on \(\mathbf{aux}\) to a single instance. The full details of zkReLU, including its optimized sumcheck protocol, are given in Appendix \ref{appendix:zkrelu}.

\section{FAC4DNN: An alternative arithmetic circuit design for modelling neural networks} \label{sec:ac}

\begin{figure*}
    \centering
    \includegraphics[width=0.9\linewidth]{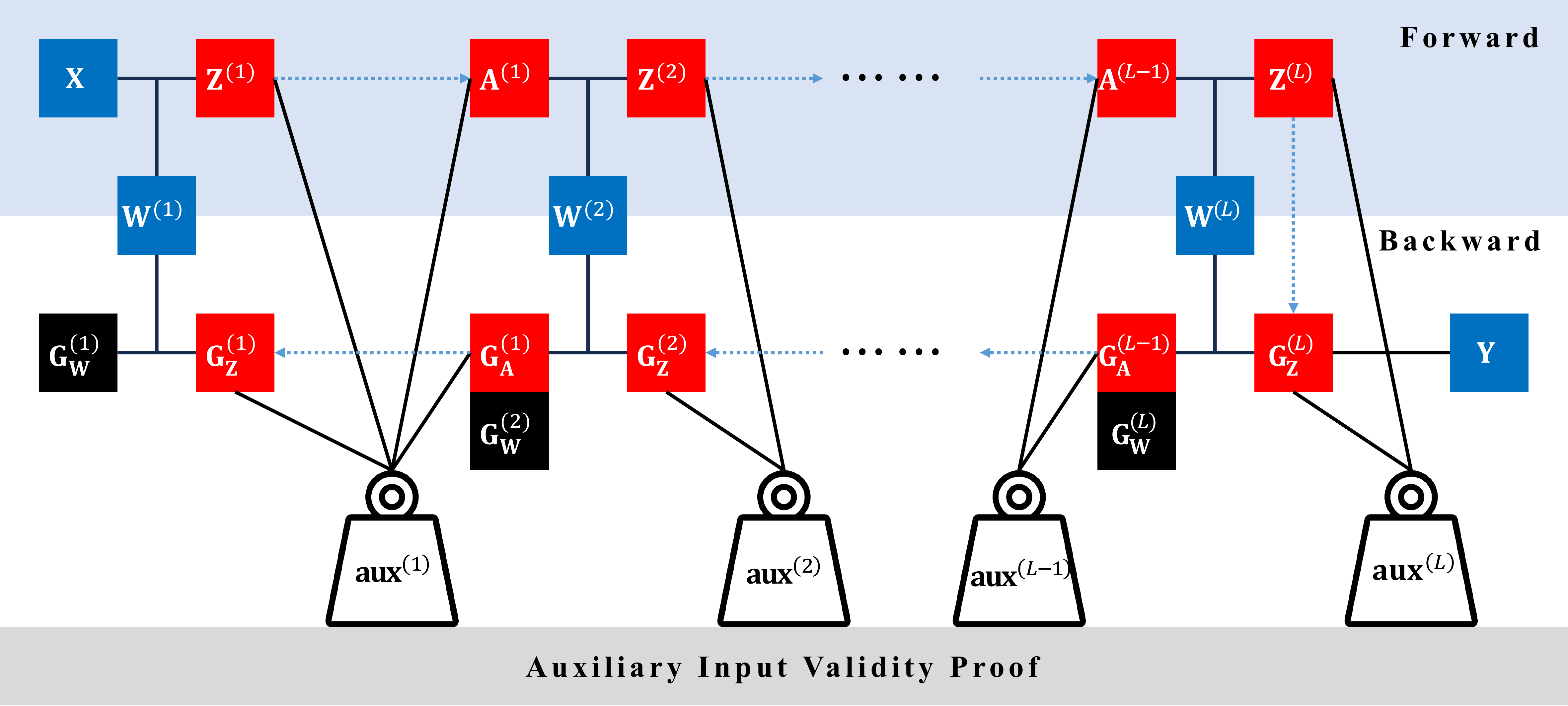}
    \caption{The configuration of FAC4DNN for each training step: The entire circuit is anchored by the auxiliary inputs $\mathbf{aux}^{(\ell)}$ through arithmetic relations represented by the black lines, which replace the non-arithmetic operations depicted by the blue dash arrows. These non-arithmetic operations include the ``comparison-with-0'' operation in ReLU and its gradient. The gradients of the model parameters are highlighted within black bounding boxes. Both the data and model parameters are bound by the Pederson commitments and are delineated within blue bounding boxes. The tensors pertinent to zkReLU are bound by the commitment of $\mathbf{aux}$ and are enclosed within red bounding boxes.}
    \label{fig:circuit}
\end{figure*}

In Section \ref{sec:zkrelu}, we operate under the assumption that the sumcheck protocols are executed on the arithmetic components of each neural network layer using the same randomness. At first glance, this might suggest potential security vulnerabilities. In the neural network structure, the output from one layer serves as the input to the subsequent layer; hence, employing identical randomness across two such layers might seemingly compromise security. However, we introduce an alternative framework, FAC4DNN, which stands for a \textbf{f}lat \textbf{a}rithmetic \textbf{c}ircuit \textbf{for} a \textbf{d}eep \textbf{n}eural \textbf{n}etwork. The core insight of FAC4DNN is the inherent separation between consecutive layers brought about by non-arithmetic operations such as ReLUs and their backpropagation. Given that tensor data on either side of this separation is secured through commitments, FAC4DNN is not constrained to maintain the traditional sequence inherent to neural networks. Instead, it can aggregate proofs across various layers and training steps, leading to notable optimizations in both the runtime and the proof sizes.

Consider the process of a single training step on a neural network (NN) using ReLU activations, which entails both a forward and backward pass. As highlighted in Figure \ref{fig:circuit}, a break in the arithmetic relations arises at each ReLU activation and its associated backward propagation. Absent the reintegration provided by zkReLU using auxiliary inputs $\mathbf{aux}^{(\ell)}$, the arithmetic operations corresponding to each NN layer $\ell$---namely the layer's input $\mathbf{A}^{(\ell-1)}$ (with $\mathbf{A}^{(0)}$ defined as the input data $\mathbf{X}$), the model parameters $\mathbf{W}^{(\ell)}$, the pre-activation $\mathbf{Z}^{(\ell)}$, and their respective gradients $\mathbf{G}_\mathbf{A}^{(\ell-1)}$, $\mathbf{G}_\mathbf{W}^{(\ell)}$, and $\mathbf{G}_\mathbf{Z}^{(\ell)}$---can be validated separately from operations in other NN layers using the sumcheck protocol.

For each training step, we operate under the assumption that both $\mathbf{W}^{(\ell)}$ and $\mathbf{G}_\mathbf{W}^{(\ell)}$ are bound by commitments. Therefore, the declared multilinear extension values for these two tensors can be directly verified through the proof of evaluation. On the other hand, the declared multilinear extension values of the remaining four tensors can initiate the proof of the correctness of ReLU and its backpropagation via zkReLU, as described in Section \ref{sec:zkdl-formulation}. Expanding our view to encompass all training steps, the proofs validating parameter updates across varying layers become independent, not only across different layers but also throughout distinct training steps. This independence is upheld by commitments that bind both old and new model parameters and their gradients. Similarly, the proofs verifying arithmetic operations and zkReLU maintain their independence across individual training steps.

The crucial insight regarding FAC4DNN is that the traditionally sequential order of the layers and training steps has been parallelized, leading to a reduction in the circuit depth by a factor of \(O(N)\), with \(N\) representing the product of the neural network depth and the number of training steps. Notably, this departure from the original sequential order allows the prover to execute the sumcheck protocols for all \(N\) parallel components of the circuit using identical randomness, without any interference between components. This concept underpins the design of the aggregated proofs over FAC4DNN.

\paragraph{Aggregating the proof} With the non-arithmetic operations transformed into auxiliary inputs and re-established as arithmetic operations, the validity of the training process becomes synonymous with the correctness of all arithmetic tensor operations across all layers. Given that tensor operations of similar sizes commonly exist in different layers and are repeated across multiple steps, proofs of these operations can be batched across both layers and steps.

Consider $N$ instances of a tensor operation that we want to aggregate, such that each instance (indexed $0\leq n \leq N-1$) with $K$ input tensors $\mathbf{X}^{(n)}_{k}$ for $0\leq k \leq K-1$, and one output tensor $\mathbf{Y}^{(n)}$. For each of these $K+1$ types of tensors, the $N$ instances of it are of the same dimensionality. By stacking the $N$ tensors of each type together, we get $\mathbf{X}_0, \mathbf{X}_1, \dots, \mathbf{X}_{K-1}$ and $\mathbf{Y}$, where using the notations of multilinear extensions, each $\me{\mathbf{X}_k}{n, \cdot}$ and $\me{\mathbf{Y}}{n, \cdot}$ can be equated with $\mathbf{X}^{(n)}_k$ and $\mathbf{Y}^{(n)}$, correspondingly. Consider any tensor operations that can be expressed as \begin{multline}\label{eq:tensor-op}
    \me{\mathbf{Y}}{n, \mathbf{i}} = \sum_{\mathbf{j}=0}^{D_\text{in}-1}f\left(\me{\mathbf{X}_0}{n, \mathbf{i}_{I_0}, \mathbf{j}_{J_0}}, \me{\mathbf{X}_1}{n, \mathbf{i}_{I_1}, \mathbf{j}_{J_1}},\right.\\ \left.\dots, \me{\mathbf{X}_{K-1}}{n, \mathbf{i}_{I_{K-1}}, \mathbf{j}_{J_{K-1}}}\right),
\end{multline} where $0\leq \mathbf{i} \leq D_\text{out} - 1, 0\leq \mathbf{j} \leq D_\text{in} - 1$ are indices expressed in binary format, $f$ is a known multivariate polynomial, and $I_k\subseteq \left[\ceil{\log_2 D_\text{out}}\right]$ and $J_k \subseteq \left[\ceil{\log_2 D_\text{in}}\right]$ are subsets of the indices defined by the nature of the operation. For simplicity, we abbreviate each summand on the right-hand side of \eqref{eq:tensor-op} as $f\left(\me{\mathbf{X}_k}{n, \mathbf{i}_{I_k}, \mathbf{j}_{J_k}}\right)_{k=0}^{K-1}$.  The sumcheck protocol for \eqref{eq:tensor-op} can be formulated as \begin{gather} 
    0 = \sum_{\mathbf{i}=0}^{D_\text{out}-1}\me{\beta}{\mathbf{u}, \mathbf{i}}\left(\me{\mathbf{Y}}{\mathbf{n}, \mathbf{i}} - \sum_{\mathbf{j}=0}^{D_\text{in}-1}f\left(\me{\mathbf{X}_k}{n, \mathbf{i}_{I_k}, \mathbf{j}_{J_k}}\right)_{k=0}^{K-1}\right)\\
    = \sum_{\mathbf{i}=0}^{D_\text{out}-1}\sum_{\mathbf{j}=0}^{D_\text{in}-1} \me{\beta}{\mathbf{u}, \mathbf{i}}\left(D_\text{in}^{-1}\me{\mathbf{Y}}{n, \mathbf{i}} - f\left(\me{\mathbf{X}_k}{n, \mathbf{i}_{I_k}, \mathbf{j}_{J_k}}\right)_{k=0}^{K-1}\right),\label{eq:tensor-op-layer}
\end{gather} for $\mathbf{u}\sim\F^{\ceil{\log_2 D_\text{out}}}$ uniformly randomly chosen by the verifier. Writing $n$ in the binary form $\mathbf{n}$, consider the weighted sum of \eqref{eq:tensor-op-layer} indexed by $\mathbf{n}$ from $0$ to $N-1$, \begin{equation}
    0 = \sum_{\mathbf{n}=0}^{N-1} \sum_{\mathbf{i} =0}^{D_\text{out}-1}\sum_{\mathbf{j} = 0}^{D_\text{in}-1}\me{\beta}{\mathbf{w} , \mathbf{n}}\me{\beta}{\mathbf{u}, \mathbf{i}}\left(D_\text{in}^{-1}\me{\mathbf{Y}}{\mathbf{n}, \mathbf{i}} - 
    f\left(\me{\mathbf{X}_k}{\mathbf{n}, \mathbf{i}_{I_k}, \mathbf{j}_{J_k}}\right)_{k=0}^{K-1}\right)\label{eq:tensor-op-aggr},
\end{equation}where $\mathbf{w}\sim\F^{\ceil{\log_2 N}}$ is the randomness chosen by the verifier, such that running the sumcheck protocol on \eqref{eq:tensor-op-aggr} proves all $N$ instances of \eqref{eq:tensor-op} simultaneously.

Alternatively, if the preceding execution of the sumcheck protocol produces a claim on the value of \(\me{\mathbf{Y}}{\mathbf{w}, \mathbf{u}}\), which is not necessarily verified against the commitment of \(\mathbf{Y}\) through the proof of evaluation, the sumcheck of \eqref{eq:tensor-op} can be reformulated as:
\begin{equation}
    \me{\mathbf{Y}}{\mathbf{w}, \mathbf{u}} = \sum_{\mathbf{n}=0}^{N-1}\sum_{\mathbf{j}=0}^{D_\text{in}-1}\me{\beta}{\mathbf{w}, \mathbf{n}}\me{\beta}{\mathbf{u}_{J_\beta}, \mathbf{j}_{J_\beta}}f\left(\me{\mathbf{X}_k}{\mathbf{n}, \mathbf{u}_{I_k}, \mathbf{j}_{J_k}}\right)_{k=0}^{K-1}, \label{eq:tensor-op-aggr-optimized}
\end{equation}
where \(J_\beta\subset\left[\ceil{\log_2 D_\text{in}}\right]\), similar to the \(I_k\) and \(J_k\), is contingent upon the nature of the tensor operations. By executing the sumcheck protocol on \eqref{eq:tensor-op-aggr-optimized}, the uncorroborated claim on the stacked \(\mathbf{Y}\) is translated to ones on \(\mathbf{X}_k\)s. These can be verified either by the proof of evaluation if \(\mathbf{X}_k\) is committed, or by invoking \eqref{eq:tensor-op-aggr-optimized} again if \(\mathbf{X}_k\) is an intermediate value that has not been committed. A comprehensive explanation of the sumcheck protocol's execution on \eqref{eq:tensor-op-aggr} and \eqref{eq:tensor-op-aggr-optimized} is provided in Appendix \ref{appendix:ac}. A crucial observation is that the sequence in which the operations appear in the original training process is \textbf{completely irrelevant} in the aggregated proof of them.

\begin{Exm} \label{ex:batch-up}
    In deep learning, matrix product $\mathbf{Y}^{(n)} = \mathbf{X}_0^{(n)}\mathbf{X}_1^{(n)}$ and Hadamard product $\mathbf{Y}^{(n)} = \mathbf{X}_0^{(n)}\odot\mathbf{X}_1^{(n)}$ are two frequently encountered tensor operations. Here, \(0\leq n \leq N-1\) indicates the indices of the operations, and we intend to batch the proofs of \(N\) operations for each type together.

As a simplification from \eqref{eq:tensor-op-layer}, these two operations can be proved by running the sumcheck protocol on the following equations:
\begin{align}
        \me{\mathbf{Y}^{(n)}}{\mathbf{u}_0, \mathbf{u}_1} &= \sum_{\mathbf{i}}\me{\mathbf{X}_0^{(n)}}{\mathbf{u}_0, \mathbf{i}}\me{\mathbf{X}_1^{(n)}}{\mathbf{i}, \mathbf{u}_1},\label{eq:exm-matmul}\\
        \me{\mathbf{Y}^{(n)}}{\mathbf{u}} &= \sum_{\mathbf{i}}\me{\beta}{\mathbf{u}, \mathbf{i}}\me{\mathbf{X}_0^{(n)}}{\mathbf{i}}\me{\mathbf{X}_1^{(n)}}{\mathbf{i}},\label{eq:exm-hadamard}
\end{align}
where \(\mathbf{u}_0, \mathbf{u}_1\) and \(\mathbf{u}\) are the random vectors compatible with the respective tensor dimensions in equations \eqref{eq:exm-matmul} and \eqref{eq:exm-hadamard}.

Following \eqref{eq:tensor-op-aggr-optimized}, the sumcheck for proof of the batched forms of the tensors can then be constructed as: \begin{align}
        \me{\mathbf{Y}}{\mathbf{w}, \mathbf{u}_0, \mathbf{u}_1} &= \sum_{\mathbf{n}}\sum_{\mathbf{i}}\me{\beta}{\mathbf{w}, \mathbf{n}}\me{\mathbf{X}_0}{\mathbf{n}, \mathbf{u}_0, \mathbf{i}}\me{\mathbf{X}_1}{\mathbf{n}, \mathbf{i}, \mathbf{u}_1},\label{eq:exm-matmul-batched}\\
        \me{\mathbf{Y}}{\mathbf{w}, \mathbf{u}} &= \sum_{\mathbf{n}}\sum_{\mathbf{i}}\me{\beta}{\mathbf{w}, \mathbf{n}}\me{\beta}{\mathbf{u}, \mathbf{i}}\me{\mathbf{X}_0}{\mathbf{n}, \mathbf{i}}\me{\mathbf{X}_1}{\mathbf{n}, \mathbf{i}},\label{eq:exm-hadamard-batched}
\end{align} where \(\mathbf{w}\sim\F^{\ceil{\log_2 N}}\) is a random vector, and can thus be proved directly without separation into $N$ proof instances, as long as the batched tensors $\mathbf{X}_0, \mathbf{X}_1, \mathbf{Y}$ are bound by commitments.
\end{Exm}

\paragraph{Re-indexing} In neural networks, repetitive tensor operations are common. Nevertheless, not every operation involves the same set of tensors. As a specific example, consider the gradient computation process. The gradient of the input to the first layer, which pertains to the data, remains uncomputed. Contrastingly, for subsequent layers, specifically for \(\ell \geq 2\), gradients of the inputs — the activations from the preceding layer — are computed based on model parameters \(\mathbf{W}^{(\ell)}\) along with other retained values of the respective layers.

Such an observation implies that, when batching proofs for these operations across the layers excluding the first, the multilinear-extension claim relates to the sequence \((\mathbf{W}^{(2)}, \mathbf{W}^{(3)}, \dots, \mathbf{W}^{(L)})\). This differs from the complete sequence \((\mathbf{W}^{(1)}, \mathbf{W}^{(2)}, \mathbf{W}^{(3)}, \dots, \mathbf{W}^{(L)})\), which finds application during forward propagation encompassing all \(L\) layers.

Such a difference highlights the criticality of re-indexing tensors, particularly when batching proofs over various layers and operations in the neural network. Appropriate indexing is imperative for the precise alignment of tensor operations, which in turn ensures the integrity and efficiency of computations along with their verifiable assertions.

Considering a stacked tensor comprising \(N\) tensors of dimension \(D\), represented as \(\mathbf{X} = (\mathbf{X}^{(0)}, \mathbf{X}^{(1)}, \dots, \mathbf{X}^{(N-1)})\), and another set of \(K\) tensors, \(\mathbf{X}_k (0\leq k \leq K-1)\), that manifest as permutations of \(\mathbf{X}\), one can represent \(\mathbf{X}_k\) as \[(\mathbf{X}^{(j_{k, 0})}, \mathbf{X}^{(j_{k, 1})}, \dots, \mathbf{X}^{(j_{k, N_k-1})}),\] where each \(j_{k, i}\) belongs to the range \([N]\). Assuming the prover makes a claim on each \(\mathbf{X}_k\), denoted as \(\me{\mathbf{X}_k}{\mathbf{u}_{k}, \mathbf{u}}\), with different randomness \(\mathbf{u}_k \sim \F^{\log_2 \ceil{N_{k}}}\) over the added dimension from stacking, and consistent randomness across other dimensions as \(\mathbf{u}\sim\F^{\ceil{\log_2D}}\). Given randomness coefficients \(r_0, r_1, \dots, r_{K-1} \sim \F\) as determined by the verifier, the relation 
\begin{align}
    \nonumber ~& \sum_{k=0}^{K-1}r_k\me{\mathbf{X}_k}{\mathbf{u}_k, \mathbf{u}} \\
    \nonumber =& \sum_{k=0}^{K-1}r_k\sum_{j=0}^{N_k-1}\me{\beta}{\mathbf{u}_k, j}\me{\mathbf{X}_k}{j, \mathbf{u}}\\
    \label{eq:sc-reindex} =& \sum_{i=0}^{N-1}\left(\sum_{k=0}^{K-1}\sum_{j=0}^{N_k-1} r_k\me{\beta}{\mathbf{u}_k, j}p_k(i, j)\right)\me{\mathbf{X}}{i, \mathbf{u}}
\end{align}
is valid, where \(p_k(i, j)\) equals 1 if the \(i\)-th component of \(\mathbf{X}\) matches the \(j\)-th element of \(\mathbf{X}_k\), and 0 otherwise (both the prover and verifier are privy to $p_k$). Consequently, by representing \(i\) in binary, executing the sumcheck protocol on \eqref{eq:sc-reindex} verifies the reordering of indices of stacked tensors from \(\mathbf{X}_k\)s to \(\mathbf{X}\), while achieving an \(O(TD)\) proving time and \(O(\log T)\) proof size.

\paragraph{Incorporating zkReLU} In Section \ref{sec:zkrelu}, we transform the correctness verification of ReLU into the validation of a tensor operation involving \(\mathbf{Z}, \mathbf{A}, \mathbf{G_A}, \mathbf{G_Z}\), and \(\mathbf{aux}\). With the incorporation of \(\mathbf{s}_{Q+R}\) and \(\mathbf{s}_Q\), this operation can be verified using the sumcheck protocol. Consequently, zkReLU seamlessly integrates into the FAC4DNN framework, enabling aggregation across layers and training steps.

Another approach to positioning zkReLU within the FAC4DNN framework is to acknowledge that the stacked tensors \(\mathbf{Z}, \mathbf{A}, \mathbf{G_A}\), and \(\mathbf{G_Z}\) are anchored by the commitment of \(\mathbf{aux}\), bypassing their individual commitments. However, to validate the non-arithmetic relations among these tensors, and upon establishing the claimed evaluations on the multilinear extension of these four tensors, the proof of evaluations on \(\mathbf{aux}\) must be augmented with the zkReLU sumcheck.

\section{Putting everything together} 

The proof produced by zkDL, employing the zkReLU protocol coupled with the compatible circuit design of FAC4DNN, delivers notable improvements in both computational and communicational efficiency for the prover and verifier, all while meeting security and privacy standards.

In this section, we assume that the training process spans \(T\) steps and the neural network includes \(L\) layers. Theoretically, maximum proof compression can be attained by amalgamating all the \(O(TL)\) repetitive units throughout the entire training process. However, such a method would necessitate the retention of all \(T\) training checkpoints. Hence, we adopt a more encompassing premise: proofs from every sequence of \(T'\) training steps are collectively processed. We proceed with the assumption that \(T'\) divides \(T\); if not, zero padding can be implemented. When \(T'=T\), it signifies that proofs from all training steps are aggregated, whereas \(T'=1\) indicates the verification of each step separately.

Moreover, in a typical neural network, layers of various types and sizes coexist, with each type linked to a distinct set of tensor operations. This diversity in operations poses challenges when aggregating proofs, even though these proofs arise from distinct layers and have been rendered independent by FAC4DNN. As an illustration, the specialized sumcheck protocols for convolution and matrix multiplication are profoundly different, complicating their aggregation. In the same vein, aggregating proofs for one large and several smaller FCs is not efficient, as this results in an undue allocation of computational resources to padded zeros.

Therefore, in this discussion, we introduce \(N_A\) families of tensor operations, denoted as \(\mathcal{A}_0, \mathcal{A}_1, \dots, \mathcal{A}_{N_A-1}\). Each family, \(\mathcal{A}_i\), includes tensor operations of analogous nature and dimensionality. For instance, matrix multiplications found in several FC layers of similar sizes or zkReLUs present at the output of a group of layers with matching dimensions fall under this category. The term \(\abs{\mathcal{A}_i}\) represents the count of operations within \(\mathcal{A}_i\) for every training step.

Similarly, we classify all tensors involved in each training step—including the training data and the auxiliary input in zkReLU—into \(N_T\) families: \(\mathcal{T}_0, \mathcal{T}_1, \dots, \mathcal{T}_{N_T-1}\). Each family contains tensors of similar dimensionality and the same nature, such as the weights of several FC layers of corresponding sizes. The symbol \(\abs{\mathcal{T}_i}\) denotes the quantity of tensors within \(\mathcal{T}_i\) during each training step.

The complete zkDL protocol involving both the prover and verifier is outlined as Protocol \ref{protocol:zkdl}. The security and overhead analyses for this protocol are presented in Sections \ref{sec:security-anal} and \ref{sec:overhead-anal}, respectively.

\begin{Exm}[FCNN]\label{exm:fcnn} We analyze a fully connected neural network (FCNN) consisting of \(L\) layers. We assume the use of ReLU activation at the output of each hidden layer and the application of the square loss. For every training step, using a data batch represented as \((\mathbf{X} = \mathbf{A}^{(0)}, \mathbf{Y})\), the correctness of the subsequent tensor operations is both necessary and sufficient to ensure the correctness of the training:

\begin{itemize}
    \item Forward propagation through each FC layer:
    \begin{align}
        \label{eq:fcnn-Z}\mathbf{Z}^{(\ell)} &= \mathbf{A}^{(\ell - 1)}\mathbf{W}^{(\ell)}, & 1\leq \ell \leq L,
    \end{align}
    \item Backward propagation of the square loss function:
    \begin{equation}
        \label{eq:fcnn-GZ-last}\mathbf{G}_{\mathbf{Z}}^{(L)} = {\mathbf{Z}^{(L)}} - \mathbf{Y},
    \end{equation}
    \item The backward propagation through each FC layer:
    \begin{align}
        \label{eq:fcnn-GA}\mathbf{G}_{\mathbf{A}}^{(\ell)} &= \mathbf{G}_\mathbf{Z}^{(\ell + 1)}{\mathbf{W}^{(\ell + 1)}}^\top, & 1 \leq \ell \leq L-1, \\
    \label{eq:fcnn-GW}\mathbf{G}_\mathbf{W}^{(\ell)} &= {\mathbf{G}_\mathbf{Z}^{(\ell)}}^\top {\mathbf{A}^{(\ell-1)}}, & 1\leq \ell \leq L,
    \end{align}
    \item ReLU and its backward propagation, as per zkReLU.
    \item Parameter updates, based on the selected optimizer.
\end{itemize}

Each type of tensor operation forms its own distinct family of tensor operations: all instances of \eqref{eq:fcnn-Z} constitute one family, all instances of \eqref{eq:fcnn-GA} another, and likewise for \eqref{eq:fcnn-GW}, zkReLU, and parameter updates. Similarly, families of tensors can be delineated based on tensors that share inherent characteristics. This applies to tensors signified by the same notation but differentiated by their superscripts, such as \( \mathbf{W}^{(1)}, \mathbf{W}^{(2)}, \ldots, \mathbf{W}^{(L)} \), which represent the parameters across all layers. With a consensus on the formation of families, the prover and verifier can employ Protocol \ref{protocol:zkdl} to generate proof for the entire training process, consolidating the proofs for different layers and training steps, thus optimizing computational and communicational expenses.

\begin{protocol*}
    \caption{zkDL}
    \label{protocol:zkdl}
    \begin{algorithmic}[1]
        \For {$t \gets 0, T', 2T', \dots, T-T'$}
            \State Prover executes training steps $t, t+1, \dots, t+T'-1$. \label{alg-line:zkdl-train}
            \For {$i \gets 0, 1, \dots N_T-1$} \Comment{$\mathbf{S}^{(t)}_i$ is the stack of \emph{all tensors in $\mathcal{T}_i$ and in training steps $t$ to $t+T'-1$}.}
                \State Prover computes commitment $\com{i}^{(t)}\gets \texttt{Commit}\left(\mathbf{S}^{(t)}_i\right)$, and sends $\com{i}^{(t)}$ to verifier. \label{alg-line:zkdl-commit}
            \EndFor
            \For {$j \gets 0, 1, \dots, N_A-1$} \Comment{$f_j^{(t)}$ is the aggregation of \emph{all operations in $\mathcal{A}_j$ and in training steps $t$ to $t+T'-1$}.}
                \State Prover and verifier execute the sumcheck protocol \eqref{eq:tensor-op-aggr} for the aggregated tensor operation $f_j^{(t)}$. \label{alg-line:zkdl-sumcheck}
                \State Prover and verifier execute the sumcheck protocol \eqref{eq:sc-reindex}, output $\me{\mathbf{S}_i^{(t)}}{\mathbf{u}_i^{(t)}}$ for each $0\leq i \leq N_T-1$. \label{alg-line:zkdl-reindex}
            \EndFor
            \For {$i \gets 0, 1, \dots N_T-1$}
                \State Prover and verifier executes the proof of evaluations for $\me{\mathbf{S}^{(t)}_i}{\mathbf{u}_i^{(t)}}$ with respect to $\com{i}^{(t)}$. \label{alg-line:zkdl-open}
            \EndFor
        \EndFor 
\end{algorithmic}
\end{protocol*}
\end{Exm}

\subsection{Security analysis} \label{sec:security-anal}

Protocol \ref{protocol:zkdl} achieves perfect completeness and near-certain soundness, ensuring that the proof's acceptance by the verifier is equivalent to the prover's adherence to the Protocol \ref{protocol:zkdl} (the training process and the proof combined) almost surely.

\begin{Thm}[Completeness]\label{thm:completeness}
    For a neural network employing the ReLU activation function, where ReLU and its backpropagation stand as the sole non-arithmetic operations, and given that the unscaled \(\mathbf{Z}^{(\ell)}\) and \(\mathbf{G}_\mathbf{A}^{(\ell)}\) are \((Q+R)\)-bit integers for the ReLU activation at every layer \(\ell\), the verifier, under the semi-honest assumption, accepts the proof with a probability of \(1\) provided the prover strictly adheres to Protocol \ref{protocol:zkdl}.
\end{Thm}

In Appendix \ref{appendix:completeness}, the validity of Theorem \ref{thm:completeness} is predicated on the perfect completeness achieved by the sumcheck protocols for zkReLU, the arithmetic operations, and the proofs of evaluations. This ensures that, for a prover fully adhering to Protocol \ref{protocol:zkdl}, all checks are passed. However, establishing the soundness of Protocol \ref{protocol:zkdl} is more intricate. The soundness of zkReLU must capture the actual correctness of the ReLU non-arithmetic operation. Additionally, the soundness of FAC4DNN must encompass the correctness of all aggregated tensor operations.

\begin{Thm}[Soundness]\label{thm:soundness}
    For a neural network employing the ReLU activation function, where ReLU and its backpropagation stand as the sole non-arithmetic operations, if any tensor operation (including ReLU) is incorrectly computed by the prover, the verifier, operating under the semi-honest assumption, accepts the proof produced by Protocol \ref{protocol:zkdl} with a probability of \(\negl{\lambda}\).
\end{Thm}

In addition to fulfilling the completeness and soundness requirements, the zkDL protocol also guarantees zero-knowledge, ensuring that it reveals no information about the training set and model parameters. This property is formalized in  Appendix \ref{appendix:zero-knowledge}.

\subsection{Overhead analysis} \label{sec:overhead-anal}

\paragraph{Prover time} Compared to general-purpose ZKP backends, zkDL does not necessitate alterations to the original structure of neural networks to accommodate the internal framework of the ZKP backends. Instead, zkDL aligns seamlessly with the inherent tensor-based architectures of the computations intrinsic to deep learning training. This alignment enables the effective utilization of the pre-existing computational environment tailored for parallel tensor computations. Furthermore, the aggregation methodologies introduced by the design of FAC4DNN overcome the limitations imposed by the sequential arrangement of neural network layers and training phases. This innovation yields a more pronounced parallelism in proof generation than in the training procedure itself. These factors combined contribute to zkDL being the first viable work on verifiable training for large neural networks.

\paragraph{Proof size} There are two primary components influencing the proof size of zkDL as delineated in Protocol \ref{protocol:zkdl}: the commitment size in Line \ref{alg-line:zkdl-commit} and the sizes of the sumcheck-based proofs in Lines \ref{alg-line:zkdl-sumcheck} and \ref{alg-line:zkdl-reindex}. For each tensor class, \(\mathcal{T}_i\), we assume that the size of every tensor within \(\mathcal{T}_i\) is \(O(\sigma_i)\). Given the square-root growth of the commitment size in Hyrax, the commitment size employing the traditional sequential generation of the proof is \(O\left(T\sum_{i=0}^{N_T-1}\abs{\mathcal{T}_i}\sqrt{\sigma_i}\right)\). Conversely, the commitment size for the aggregated tensors is \(O\left(\frac{T}{T'}\sum_{i=0}^{N_T-1}\sqrt{T'\abs{\mathcal{T}_i}\sigma_i}\right)\), which simplifies to \(O\left(\frac{T}{\sqrt{T'}}\sum_{i=0}^{N_T-1}\sqrt{\abs{\mathcal{T}_i} \sigma_i}\right)\). 

On the other hand, for the sumcheck protocols, proof sizes are generally logarithmic with respect to the complexity of the operation. We represent the complexity of a singular operation in each \(\mathcal{A}_i\) as \(O(\psi_i)\), rendering the proof size of one operation as \(O(\log \psi_i)\). With the traditional sequential proof, the cumulative verification cost would, in a straightforward manner, accumulate to \(O\left(T\sum_{i=0}^{N_A-1}\abs{\mathcal{A}_i}\log\psi_i\right)\). However, when using FAC4DNN for proof aggregation, the multiplicative factor of the total aggregated proof instances is replaced by a nearly negligible additive term that is logarithmic in this number. Hence, the refined proof size is \(O\left(\frac{T}{T'}\left(N_A \log{T'} + \sum_{i=0}^{N_A-1}\left(\log\abs{\mathcal{A}_i} + \log \psi_i\right)\right)\right)\).

\paragraph{Verifier time} The analysis of the verifier's time closely mirrors the analysis of the proof size. With Hyrax \cite{hyrax}, the verifier's time in the proof of evaluation scales as the square root of the committed tensor's size. Consequently, the verifier's time for each tensor class \(\mathcal{T}_i\) is reduced by a factor of \(O\left(\sqrt{T'|\mathcal{T}_i|}\right)\). On the other hand, given the logarithmic nature of the verifier's time in the sumcheck protocol, the verifier's time for each tensor class \(\mathcal{A}_i\) experiences a reduction almost by a factor of \(O\left(T'|\mathcal{A}_i|\right)\).

\section{Experiments} \label{sec:experiments}

We developed zkDL in CUDA. Our \textbf{open-source implementation} of zkDL is available at \url{https://github.com/jvhs0706/zkdl-train}. To ensure seamless interfacing with tensor structures and their respective specialized aggregated proofs, our approach entailed creating customized CUDA kernels. These kernels cater to tensors, deep learning layers, and the novel cryptographic protocols presented in this research. Our framework builds upon \texttt{ec-gpu} \cite{ec-gpu}, which is a CUDA implementation of the BLS12-381 curve, guaranteeing 128-bit security. The efficacy of zkDL was assessed using Example \ref{exm:fcnn}. Given the challenges posed by quantization-induced rounding errors, we implemented a scaling factor of \(2^{16}\). This assured that every real-number computation within the system was encapsulated in the interval \([-2^{15}, 2^{15})\). Consequently, these numbers were aptly scaled as 32-bit integers. It is pivotal to highlight that throughout our experimental iterations, overflow incidents were conspicuously absent. Our experimental evaluations were orchestrated on a computing node nestled within a cluster, equipped with a Tesla A100 GPU. 

\paragraph{Baselines}
Earlier research on zero-knowledge verifiable inference did not offer ample techniques to incorporate verifiability into the training phase. In contrast, groundbreaking studies on zero-knowledge verifiable training, such as \cite{veriml, unlearning, unlearning-2}, predominantly concentrated on classical machine learning models, leaving real-world scale deep neural networks untouched. Nevertheless, for a comprehensive illustration of the benefits of aggregating proofs across multiple layers and training steps, we positioned zkDL against the traditional sequential proof generation found in the GKR protocol, which fully adheres to the reverse sequence of training steps, as well as the order of layers navigated throughout the training procedure.

\begin{figure*}[htbp]
    \centering
    \includegraphics[width=\linewidth]{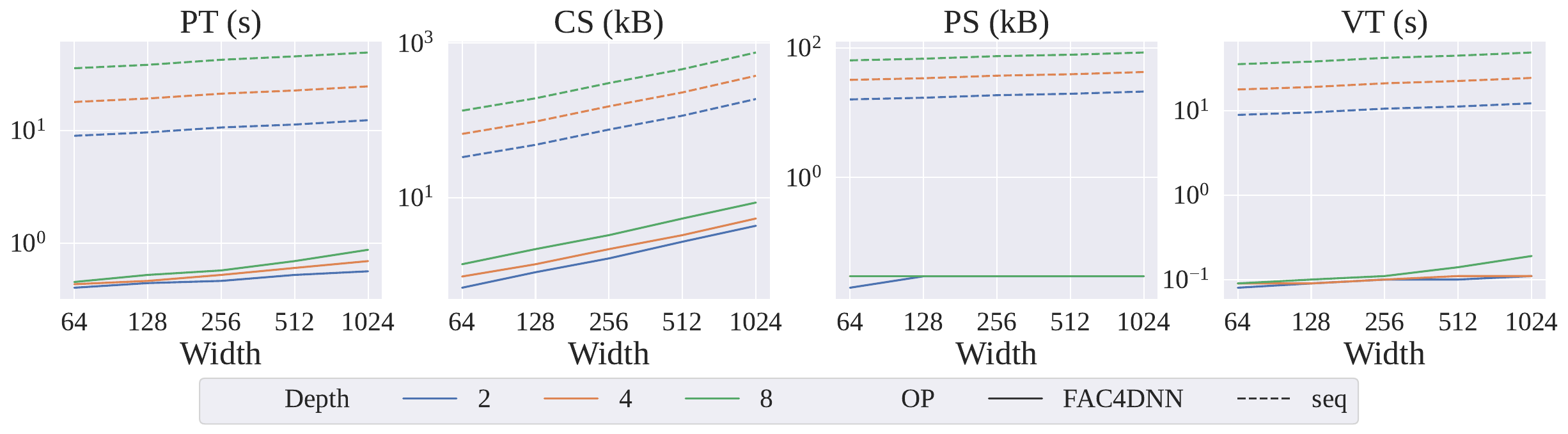}
    \caption{zkDL's performance under different NN sizes: OP (order of proving), PT (proving time), CS (commitment size), PS (proof size), VT (verifying time). } 
    \label{fig:additional-exp}
\end{figure*}

\paragraph{Power of FAC4DNN} To showcase the efficiency of FAC4DNN, which liberates us from the mandate of strictly adhering to the sequential order of computations during proof generation, we embarked on a comparative analysis between zkDL and less compact proof variations. Specifically, we juxtaposed it against \textbf{1)} the proof that unwaveringly conforms to a reverse computation sequence, and \textbf{2)} a proof that only undergoes aggregation within individual training steps. 

The experimental trials were executed on the CIFAR-10 dataset, utilizing an 8-hidden-layer perceptron, activated by ReLU, comprising 1,024 neurons in each layer. Notably, the input and output layers diverged in size, hosting 3,072 and 10 neurons, respectively. Therefore, the number of parameters in the neural network is as large as 10 millions. For the sake of comprehensiveness, our approach was to modify the aggregation span across different numbers of steps and subsequently record the findings, as tabulated in Table \ref{tab:results}.

\begin{table}
\caption{Per-step computational and communicational costs of zkDL: \(T'\) (number of aggregated training steps; \tt{seq} = no aggregation, 1 = within-step aggregation), PT (proving time), CS (commitment size), PS (proof size), VT (verifying time).} \label{tab:results}
\centering
\begin{tabular}{cccccc}
\hline
$T'$                  & BS & PT (s) & CS (kB) & PS (kB) & VT (s) \\ \hline
\multirow{3}{*}{\tt{seq}}  & 16 & 45     & 510     & 78      & 45     \\
                      & 32 & 47     & 610     & 82      & 46     \\
                      & 64 & 49     & 750     & 85      & 46     \\ \hline
\multirow{3}{*}{1}    & 16 & 6.0    & 230     & 11      & 6.0    \\
                      & 32 & 6.1    & 260     & 11      & 6.2    \\
                      & 64 & 6.2    & 270     & 12      & 6.3    \\ \hline
\multirow{3}{*}{4}    & 16 & 1.8    & 110     & 3.1     & 1.8    \\
                      & 32 & 1.8    & 120     & 3.3     & 1.9    \\
                      & 64 & 1.9    & 120     & 3.4     & 1.9    \\ \hline
\multirow{3}{*}{16}   & 16 & 0.83   & 56      & 0.81    & 0.80   \\
                      & 32 & 0.83   & 60      & 0.86    & 0.82   \\
                      & 64 & 0.84   & 61      & 0.88    & 0.84   \\ \hline
\multirow{3}{*}{64}   & 16 & 0.84   & 28      & 0.31    & 0.53   \\
                      & 32 & 0.83   & 30      & 0.32    & 0.52   \\
                      & 64 & 0.85   & 30      & 0.34    & 0.56   \\ \hline
\multirow{3}{*}{256}  & 16 & 0.84   & 14      & 0.085   & 0.30   \\
                      & 32 & 0.84   & 15      & 0.089   & 0.33   \\
                      & 64 & 0.84   & 15      & 0.091   & 0.33   \\ \hline
\multirow{3}{*}{1,024} & 16 & 0.85   & 8.2     & 0.030   & 0.18   \\
                      & 32 & 0.85   & 8.6     & 0.032   & 0.19   \\
                      & 64 & 0.86   & 8.7     & 0.033   & 0.19   \\ \hline
\end{tabular}
\end{table}

It is pertinent to highlight that in a bid to provide a more lucid perspective on the superiority of zkDL, we have presented the computational and communication expenditures averaged \emph{per-step}, as opposed to the cumulative costs associated with aggregated units of diverse step counts.

The data presented in Table \ref{tab:results} reveals a discernible trend: as the number of aggregated training steps increases, various performance metrics of zkDL generally improve. Notably, while FAC4DNN does not offer improvements in the theoretical complexity concerning proving time, the parallel processing capabilities of CUDA still manage to deliver a marked reduction in real-world proving durations. This improvement continues until it plateaus in less than 1 second. This limitation is due to CUDA's memory constraints, after which point data transfers become necessary within every aggregated proof. Furthermore, with sound theoretical underpinnings, as the number of aggregated steps increases, the sizes of the commitments and proofs, as well as the verifier times, decrease significantly, thanks to the design of FAC4DNN.


\paragraph{Further experimental results} In Figure \ref{fig:additional-exp}, we evaluate the efficiency of zkDL when applied to neural networks of different dimensions. Keeping the number of aggregated steps constant at 1,024 and the batch size at 64, which is consistent with the largest experimental setting shown in Table \ref{tab:results}, we modify the neural network's width and depth. The outcomes are then contrasted with the fully sequential proof.

In alignment with the data in Table \ref{tab:results}, zkDL consistently outperforms the sequential proof in all facets. An unusual observation from both Table \ref{tab:results} and Figure \ref{fig:additional-exp} is that the verifying time is not considerably shorter than the proving time, a deviation from what is typically observed in most generic ZKP frameworks. This anomaly can be explained by the verification protocol's predominantly sequential nature, preventing it from fully harnessing the parallel computational capabilities of CUDA. Hence, the development of verification algorithms that are both theoretically and empirically expedient, particularly those optimized for parallel computing environments like CUDA, may forge the path for future advancements in specialized ZKPs for deep learning training.

\section{Related work} 

\paragraph{Verifiable machine learning inference} Zero-knowledge proof (ZKP) systems have emerged as important solutions to address security and privacy concerns in machine learning. These systems enable the verification of machine learning inference correctness without disclosing the underlying data or model. Notably, zkCNN \cite{zkcnn} introduced an interactive proof protocol for convolutional layers, based on the GKR protocol \cite{gkr} and its refinements \cite{libra,DBLP:conf/ccs/ZhangLWZSXZ21,orion}. This solution provides zero-knowledge verifiable inference for VGG-scale convolutional neural networks, expanding verifiable computations to modern deep learning. Meanwhile, zk-SNARK-based inference, represented by ZEN \cite{zen}, vCNN \cite{vcnn}, pvCNN \cite{pvcnn}, and ZKML \cite{zkml}, Mystique \cite{mystique}, ezDPS \cite{ezdps}, focuses on enhancing the compatibility of neural networks with the zk-SNARK backend \cite{snark1, snark2, snark3, snark4, hyrax, snark6}, scaling up non-interactive zero-knowledge inference. Once the committed model is verified to be correctly trained using this work, the verifiable inference can serve as a downstream application. 

\paragraph{Verifiable machine learning training} VeriML \cite{veriml} serves as an initial endeavour in zero-knowledge verifiable training for core machine learning algorithms. Building upon this, works on provable unlearning \cite{unlearning, unlearning-2} offer proofs of correct training execution and updates in machine unlearning, advancing past the probabilistic verifications in \cite{vul1, vul2, vul3}. Their methodology, however, is restricted to instances where users provide data and fully govern data changes via explicit requests. This restricts proof generation to only data points involved in the aforementioned requests, leaving it ill-suited for our setting where the training data is also considered the intellectual property of the prover. Additionally, the models they support are constrained to limited-scale neural networks while bypassing certain non-arithmetic operations, like replacing ReLU with square activation. Another noteworthy contribution is zkPoT \cite{zkpot}, which, based on MPC-in-the-head \cite{mpc-in-the-head}, offers zero-knowledge verifiable logistic regression.

\paragraph{Proof of learning (PoL) \cite{pol}} PoL serves as a non-cryptographic-based alternative to verifiable training. However, its probabilistic guarantees render it unsuitable for legitimacy-related settings like zkDL \cite{pol-limits, pol-spoof}. Additionally, its threat model assumes adversaries to forge proofs by expending less computation resources than training, which does not deter dedicated malicious prover capable of deviating from the prescribed training logic (e.g., planting backdoors) at the cost of equivalent or additional computational power.

\section{Conclusion}
This paper introduces zkDL, the inaugural specialized zero-knowledge proof system tailored for deep learning training. By harnessing the unique computational structure of deep neural networks that enables a \emph{compressed proof} across data batch, training step and neural network layers, zkDL substantially diminishes the time and communication overheads involved in verifying the genuine execution of deep learning training. Complemented by our pioneering CUDA-based implementation for verifiable deep learning, zkDL adeptly addresses the authenticity concerns related to trained neural networks with \emph{provable security guarantees}. Experimentally, zkDL enables the generation of complete and sound proofs in less than a second per batch update for an 8-layer neural network with 10M parameters and a batch size of 64, while provably ensuring the privacy of data and model parameters. To our best knowledge, we are not aware of any existing work on zero-knowledge proof of deep learning training that is scalable to million-size networks.

\bibliographystyle{plain}
\bibliography{reference.bib}

\appendices

\section{zkReLU} \label{appendix:zkrelu}

In this appendix, we delve into the specifics of zkReLU. The foundational requirement of zkReLU, stemming from the sumcheck-protocol-based tensor operations discussed in Section \ref{sec:ac}, encompasses the declared evaluations on both the input and output of ReLU's forward and backward propagations. Specifically, these are \(\me{\mathbf{A}}{\mathbf{u_A}}, \me{\mathbf{Z}}{\mathbf{u_Z}}, \me{\mathbf{G_Z}}{\mathbf{u_{G_Z}}}, \me{\mathbf{G_A}}{\mathbf{u_{G_A}}}\).

The sumcheck protocol for the forward propagation, denoted as \eqref{eq:aux-Z} and \eqref{eq:aux-A}, is given by
\begin{gather}
    \me{\mathbf{Z}}{\mathbf{u_Z}} = \sum_{\mathbf{i} = 0}^{D-1} \sum_{\mathbf{j} = 0}^{Q+R-1}\me{\beta}{\mathbf{u_Z}, \mathbf{i}}\me{\mathbf{aux}}{0, \mathbf{i}, \mathbf{j}}\me{\mathbf{s}_{Q+R}}{\mathbf{j}},\label{eq:zkrelu-forward-in-sc}\\
    \me{\mathbf{A}}{\mathbf{u_A}} = \sum_{\mathbf{i} = 0}^{D-1} \sum_{\mathbf{j} = 0}^{Q+R-1}\me{\beta}{\mathbf{u_A}, \mathbf{i}}\left(1-\me{\mathbf{aux}}{0, \mathbf{i}, Q+R-1}\right)\me{\mathbf{aux}}{0, \mathbf{i}, \mathbf{j}}\me{\mathbf{s}'}{\mathbf{j}},\label{eq:zkrelu-forward-out-sc}
\end{gather}
where 
\[\mathbf{s'} := \left(0, 0, 0, \dots, 1, 1, 2, 4, \dots 2^{Q-2}, -2^{Q-1}\right)^\top.\]
Additionally, the AIVP is expressed as
\begin{equation}
    0 = \sum_{\mathbf{i} = 0}^{D-1}\sum_{\mathbf{j} = 0}^{Q+R-1} \me{\beta}{\mathbf{u_\text{bin}}, \mathbf{i}\oplus \mathbf{j}} \left(\me{\mathbf{aux}}{0, \mathbf{i}, \mathbf{j}}^2 - \me{\mathbf{aux}}{0, \mathbf{i}, \mathbf{j}}\right),\label{eq:zkrelu-forward-bin-sc}
\end{equation}
with $\mathbf{u}_\text{bin}\sim\F^{\ceil{\log_2D} + \ceil{\log_2 (Q+R)}}$ being uniformly randomly drawn by the prover.

Analogously, mirroring the construction above for backpropagation: \begin{gather}
    \me{\mathbf{G_A}}{\mathbf{u_{G_A}}} = \sum_{\mathbf{i} = 0}^{D-1} \sum_{\mathbf{j} = 0}^{Q+R-1}\me{\beta}{\mathbf{u_{G_A}}, \mathbf{i}}\me{\mathbf{aux}}{1, \mathbf{i}, \mathbf{j}}\me{\mathbf{s}_{Q+R}}{\mathbf{j}},\label{eq:zkrelu-backward-in-sc}\\
    \me{\mathbf{G_Z}}{\mathbf{u_{G_Z}}} = \sum_{\mathbf{i} = 0}^{D-1} \sum_{\mathbf{j} = 0}^{Q+R-1}\me{\beta}{\mathbf{u_{G_Z}}, \mathbf{i}}\left(1-\me{\mathbf{aux}}{0, \mathbf{i}, Q+R-1}\right)\me{\mathbf{aux}}{1, \mathbf{i}, \mathbf{j}}\me{\mathbf{s}'}{\mathbf{j}},\label{eq:zkrelu-backward-out-sc}\\
    0 = \sum_{\mathbf{i} = 0}^{D-1}\sum_{\mathbf{j} = 0}^{Q+R-1} \me{\beta}{\mathbf{u_\text{bin}}, \mathbf{i}\oplus \mathbf{j}} \left(\me{\mathbf{aux}}{1, \mathbf{i}, \mathbf{j}}^2 - \me{\mathbf{aux}}{1, \mathbf{i}, \mathbf{j}}\right).\label{eq:zkrelu-backward-bin-sc}
\end{gather}

To optimize the verification process, the verifier introduces randomness \( r, r' \sim \F \) to compress the six sumcheck equations, from \eqref{eq:zkrelu-forward-in-sc} to \eqref{eq:zkrelu-backward-bin-sc}. Each equation is multiplied by weights \( r^2, r, 1, r'r^2, r'r, \) and \( r' \) respectively, and subsequently aggregated. This compression technique leads to a succinct proof representation requiring only \( 3\log_2\left(D(Q+R)\right) + O(1) \) field elements.

This procedure creates three intermediate claims about \( \mathbf{aux} \): \( \me{\mathbf{aux}}{0, \mathbf{v}, \mathbf{w}} \), \( \me{\mathbf{aux}}{1, \mathbf{v}, \mathbf{w}} \), and \( \me{\mathbf{aux}}{0, \mathbf{v}, Q+R-1} \) for $\mathbf{v}\sim \F^{\log_2D}, \mathbf{w}^{\log_2 \left(Q+R\right)}$ chosen due to the randomness during the execution of the sumcheck. However, they can be merged into a singular claim, further reducing the proof size to \( 2\log_2 (Q+R) + O(1) \). Consequently, the entire zkReLU proof compression is succinctly captured as \( 3\log_2 D + 5\log_2 (Q+R) + O(1) \) field elements, with one additional proof of evaluation on \( \mathbf{aux} \).

\section{FAC4DNN} \label{appendix:ac}

The aggregated proof for tensor operations, specifically the sumcheck on Equation \eqref{eq:tensor-op-aggr}, follows Protocol \ref{protocol:tensor-op-aggr}. In the same vein, the aggregation of specialized optimized sumcheck protocols, as given by Equation \eqref{eq:tensor-op-aggr-optimized}, adheres to Protocol \ref{protocol:tensor-op-aggr-optimized}. In both instances, compression begins over the additional axis resulting from stacking. This process methodically whittles down the proof over the stacked tensor to a singular tensor as shown in \eqref{eq:tensor-op-compressed} and \eqref{eq:tensor-op-compressed-optimized}. Subsequently, the sumcheck protocol designed for individual tensor operations is directly implemented.

\begin{protocol*}
    \caption{Sumcheck on Equation \eqref{eq:tensor-op-aggr}}
    \label{protocol:tensor-op-aggr}
    \begin{algorithmic}[1]
        \Require Prover $\mathcal{P}$, verifier $\mathcal{V}$, $N, D_\text{in}, D_\text{out}$ are powers of $2$ (zero-padding may be applied otherwise)
        \State Denote $n:= \log_2 N, d_\text{in} := \log_2 D_\text{in}, d_\text{out} := \log_2 D_\text{out}$
        \State $\mathcal{P}$ sends to verifier the uni-variate polynomial \begin{equation}
            f_0(v) := \sum_{\mathbf{n}'\in \{0, 1\}^{n-1}} \sum_{\mathbf{i} =0}^{D_\text{out}-1}\sum_{\mathbf{j} = 0}^{D_\text{in}-1}\me{\beta}{\mathbf{w}_{[1:n]}, \mathbf{n}'}\me{\beta}{\mathbf{u}, \mathbf{i}}\left(D_\text{in}^{-1}\me{\mathbf{Y}}{v, \mathbf{n}', \mathbf{i}} - 
            f\left(\me{\mathbf{X}_k}{v, \mathbf{n}', \mathbf{i}_{I_k}, \mathbf{j}_{J_k}}\right)_{k=0}^{K-1}\right)
        \end{equation}
        \State $\mathcal{V}$ computes $g_0(v)\gets \me{\beta}{\mathbf{w}_{[0]}, v} f_0(v)$, and checks $g_0(0) + g_0(1) = 0$
        \State $\mathcal{V}$ sends to $\mathcal{P}$ a uniform random $v_0\sim \F$
        \For{$t\gets 1, 2, \dots, n-1$}
            \Comment{Denote $\mathbf{v}_t:= \left(v_0, v_1, \dots, v_{t-1}\right)^\top$}
            \State $\mathcal{P}$ sends to $\mathcal{V}$ the uni-variate polynomial \begin{equation}
                f_t(v) := \sum_{\mathbf{n}'\in \{0, 1\}^{n-t-1}} \sum_{\mathbf{i} =0}^{D_\text{out}-1}\sum_{\mathbf{j} = 0}^{D_\text{in}-1}\me{\beta}{ \mathbf{w}_{[t+1:n]}, \mathbf{n}'}\me{\beta}{\mathbf{u}, \mathbf{i}}\left(D_\text{in}^{-1}\me{\mathbf{Y}}{\mathbf{v}_t, v, \mathbf{n}', \mathbf{i}} - 
                f\left(\me{\mathbf{X}_k}{\mathbf{v}_t, v, \mathbf{n}', \mathbf{i}_{I_k}, \mathbf{j}_{J_k}}\right)_{k=0}^{K-1}\right)
        \end{equation}
        \State $\mathcal{V}$ computes $g_t(v)\gets \me{\beta}{\mathbf{w}_{[t]}, v} f_t(v)$, and checks $g_t(0) + g_t(1) = f_{t-1}(v_{t-1})$
        \State $\mathcal{V}$ sends to $\mathcal{P}$ a uniform random $v_t\sim \F$
        \EndFor
        \State $\mathcal{P}$ and $\mathcal{V}$ execute the sumcheck protocol on \label{alg-line:tensor-op-compressed} \begin{equation}
            f_{n-1}(v_{n-1}) = \sum_{\mathbf{i} =0}^{D_\text{out}-1}\sum_{\mathbf{j} = 0}^{D_\text{in}-1}\me{\beta}{\mathbf{u}, \mathbf{i}}\left(D_\text{in}^{-1}\me{\mathbf{Y}}{\mathbf{v}_{n-1}, \mathbf{i}} - 
                f\left(\me{\mathbf{X}_k}{\mathbf{v}_{n-1}, \mathbf{i}_{I_k}, \mathbf{j}_{J_k}}\right)_{k=0}^{K-1}\right) \label{eq:tensor-op-compressed}
        \end{equation}
    \end{algorithmic}
\end{protocol*}

\begin{protocol*}
    \caption{Sumcheck on Equation \eqref{eq:tensor-op-aggr-optimized}}
    \label{protocol:tensor-op-aggr-optimized}
    \begin{algorithmic}[1]
        \Require Prover $\mathcal{P}$, verifier $\mathcal{V}$, $N, D_\text{in}, D_\text{out}$ are powers of $2$ (zero-padding may be applied otherwise)
        \State Denote $n:= \log_2 N, d_\text{in} := \log_2 D_\text{in}, d_\text{out} := \log_2 D_\text{out}$ 
        \State $\mathcal{P}$ sends to verifier the uni-variate polynomial \begin{equation}
            f_0(v) := \sum_{\mathbf{n}'\in \{0, 1\}^{n-1}} \sum_{\mathbf{j}=0}^{D_\text{in}-1}\me{\beta}{\mathbf{w}_{[1:n]}, \mathbf{n}'}\me{\beta}{\mathbf{u}_{J_\beta}, \mathbf{j}_{J_\beta}}f\left(\me{\mathbf{X}_k}{v, \mathbf{n}', \mathbf{u}_{I_k}, \mathbf{j}_{J_k}}\right)_{k=0}^{K-1}
        \end{equation}
        \State $\mathcal{V}$ computes $g_0(v)\gets \me{\beta}{\mathbf{w}_{[0]}, v} f_0(v)$, and checks $g_0(0) + g_0(1) = \me{\mathbf{Y}}{\mathbf{w}, \mathbf{u}}$
        \State $\mathcal{V}$ sends to $\mathcal{P}$ a uniform random $v_0\sim \F$
        \For{$t\gets 1, 2, \dots, n-1$}
            \Comment{Denote $\mathbf{v}_t:= \left(v_0, v_1, \dots, v_{t-1}\right)^\top$}
            \State $\mathcal{P}$ sends to $\mathcal{V}$ the uni-variate polynomial \begin{equation}
                f_t(v) := \sum_{\mathbf{n}'\in \{0, 1\}^{n-t-1}} \sum_{\mathbf{j}=0}^{D_\text{in}-1}\me{\beta}{\mathbf{w}_{[1:n]}, \mathbf{n}'}\me{\beta}{\mathbf{u}_{J_\beta}, \mathbf{j}_{J_\beta}}f\left(\me{\mathbf{X}_k}{\mathbf{v}_t, v, \mathbf{n}', \mathbf{u}_{I_k}, \mathbf{j}_{J_k}}\right)_{k=0}^{K-1}
        \end{equation}
        \State $\mathcal{V}$ computes $g_t(v)\gets \me{\beta}{\mathbf{w}_{[t]}, v} f_t(v)$, and checks $g_t(0) + g_t(1) = f_{t-1}(v_{t-1})$
        \State $\mathcal{V}$ sends to $\mathcal{P}$ a uniform random $v_t\sim \F$
        \EndFor
        \State $\mathcal{P}$ and $\mathcal{V}$ execute the sumcheck protocol on \label{alg-line:tensor-op-compressed-optimized} \begin{equation}
            f_{n-1}(v_{n-1}) = \sum_{\mathbf{j}=0}^{D_\text{in}-1}\me{\beta}{\mathbf{u}_{J_\beta}, \mathbf{j}_{J_\beta}}f\left(\me{\mathbf{X}_k}{\mathbf{v}_{n-1}, \mathbf{u}_{I_k}, \mathbf{j}_{J_k}}\right)_{k=0}^{K-1} \label{eq:tensor-op-compressed-optimized}
        \end{equation}
    \end{algorithmic}
\end{protocol*}

\section{Security analysis} \label{appendix:pet}

\subsection{Completeness}\label{appendix:completeness} 
\begin{proof}[Proof sketch of Theorem \ref{thm:completeness}]
    Note that there are three scenarios in which a semi-honest verifier might reject a proof: the sumchecks in Lines \ref{alg-line:zkdl-sumcheck}, \ref{alg-line:zkdl-reindex}, and the proof of evaluations in Line \ref{alg-line:zkdl-open}. 

    For the sumcheck concerning aggregated tensor operations in Line \ref{alg-line:zkdl-sumcheck}, as specified in Protocol \ref{protocol:tensor-op-aggr} or \ref{protocol:tensor-op-aggr-optimized}, regardless of which is invoked, the polynomials \(f_t\), extended by the verifier to \(g_t\), ensure that the sum \(g_t(0)+g_t(1)\) comprehensively retrieves the sum from the preceding round. As such, given that the prover adheres to the protocol, it consistently stands that \(g_t(0) + g_t(1)\) matches the preceding claim. This ensures that \(f_t\) is accepted by the verifier in round \(t\) with probability 1. Lastly, owing to the perfect completeness of the sumchecks for singular tensor operations, the verifier will also unconditionally accept the proof presented by an honest prover. Once an honest prover has successfully navigated the prior \(n\) rounds and reduced the \(N\) operations to one, the probability that the sumcheck protocol for this isolated operation also fails is zero.

    Moreover, Equation \eqref{eq:sc-reindex} presents as the inner product of two vectors with length \(N-1\), specifically the public input \(\left(\sum_{k=0}^{K-1}\sum_{j=0}^{N_k-1} r_k\me{\beta}{\mathbf{u}_k, j}p_k(i, j)\right)_{i=0}^{N-1}\) and the private input \(\left(\me{\mathbf{X}}{i, \mathbf{u}}\right)_{i=0}^{N-1}\). As a result, the completeness in this phase arises from the perfect completeness of the sumcheck for vector inner-products.

    Finally, due to the perfect completeness of the proof of evaluations, and given the declared evaluations — that is, \(\me{\mathbf{S}_i^{(t)}}{\mathbf{u}_i^{(t)}}\) — are correct, the semi-honest verifier accepts the proof of evaluations on all stacked tensors with certainty in Line \ref{alg-line:zkdl-open}.
\end{proof}

\subsection{Soundness}\label{appendix:soundness} The proof of soundness (Theorem \ref{thm:soundness}) relies on Lemma \ref{lem:soundness-zkrelu} that captures the soundness of zkReLU: 

\begin{Lem}[Soundness of zkReLU] \label{lem:soundness-zkrelu}
    If the equalities \eqref{eq:aux-Z}, \eqref{eq:aux-GA}, \eqref{eq:aux-A}, \eqref{eq:aux-GZ} and \eqref{eq:aux-bin} hold, then the forward and backward propagation of the ReLU activation is correctly computed as: \begin{align}
            \label{eq:soundness-zkrelu-forward}\mathbf{A} &=  \left\lfloor\frac{\1\left\{\mathbf{Z}\geq 0\right\}\odot\mathbf{Z}}{2^R}\right\rceil, \\
            \label{eq:soundness-zkrelu-backward}\mathbf{G}_\mathbf{Z} &=  \left\lfloor\frac{\1\left\{\mathbf{Z} \geq 0\right\}\odot\mathbf{G_A}}{2^R}\right\rceil.
        \end{align}
\end{Lem}

\begin{proof}[Proof of Lemma \ref{lem:soundness-zkrelu}]
    Consider first the correctness of \eqref{eq:aux-bin}. This ensures that $\mathbf{aux} \in \{0, 1\}^{2\times D\times (Q+R)}$. Following from \eqref{eq:aux-Z}, we derive \begin{equation}
        \sum_{j=0}^{Q+R-2}2^j\mathbf{aux}_{[0, :, j]} - 2^{Q+R-1} \mathbf{aux}_{[0, :, Q+R-1]} = \mathbf{Z},\label{eq:aux-Z-decomp}
    \end{equation}
    where the first term on the left-hand side of \eqref{eq:aux-Z-decomp} has a bound between $0$ and $2^{Q+R-1} -1$. Consequently, the second term defines the sign of $\mathbf{Z}$, giving $\1\left\{\mathbf{Z}\geq 0\right\} = 1 - \mathbf{aux}_{[0, Q+R-1, :]}$. Applying a similar logic to \eqref{eq:aux-GA}, we deduce that $\1\left\{\mathbf{G_A}\geq 0\right\} = 1- \mathbf{aux}_{[1, Q+R-1, :]}$. 
    
    For the rescaling of $\mathbf{Z}$, we can represent it as: \begin{equation}
        \round{\frac{\mathbf{Z}}{2^R}} = \round{\frac{\sum_{j=0}^{R-1}2^j \mathbf{aux}_{[0,:,j]}}{2^R}} + \sum_{j=0}^{Q-1}2^j \mathbf{aux}_{[0,:,R+j]} - 2^Q\mathbf{aux}_{[0,:,Q+R-1]}.
    \end{equation}
    For each dimension denoted by $i$, the first term on the right-hand side is $1$ precisely when $\mathbf{aux}_{[0,i,R-1]} = 1$. Thus, it is valid that $\round{\frac{\mathbf{Z}}{2^R}} = \mathbf{aux}_{[0, :, R:Q+R]}\mathbf{s}_{Q} + \mathbf{aux}_{[0, :, R-1]}$. From \eqref{eq:aux-A}, we derive:
    \begin{equation}
        \mathbf{A} = \1\left\{\mathbf{Z}\geq 0\right\}\odot \round{\frac{\mathbf{Z}}{2^R}} = \round{\frac{\1\left\{\mathbf{Z}\geq 0\right\}\odot\mathbf{Z}}{2^R}},
    \end{equation}
    which leads to \eqref{eq:soundness-zkrelu-forward}. Using the same reasoning on \eqref{eq:aux-GZ} confirms the validity of \eqref{eq:soundness-zkrelu-backward}.
\end{proof}

Lemma \ref{lem:soundness-zkrelu} asserts that, through the integration of $\mathbf{aux}$ and its corresponding auxiliary components, the correctness of ReLU can be equated to the correctness of conventional arithmetic tensor operations. This enables aggregation across layers and training iterations within Protocol \ref{protocol:zkdl}. Consequently, for soundness, it is only necessary to ensure the correctness of all tensor operations.

\begin{proof}[Proof of Theorem \ref{thm:soundness}]
    A malicious prover may attempt to cheat in Protocol \ref{protocol:zkdl} specifically at Lines \ref{alg-line:zkdl-sumcheck}, \ref{alg-line:zkdl-reindex}, and \ref{alg-line:zkdl-open}. We investigate the implications in the reverse order of these steps.

    Beginning with the proof of evaluation in Line \ref{alg-line:zkdl-open}, the commitment scheme's binding properties ensure a soundness error of \(\negl{\lambda}\). Thus, if the prover attempts to falsely claim that \(v\) is equivalent to \(\me{\mathbf{S}_i^{(t)}}{\mathbf{u}_i^{(t)}}\) for any \(v\neq \me{\mathbf{S}_i^{(t)}}{\mathbf{u}_i^{(t)}}\), the verification passes with a probability no greater than \(\negl{\lambda}\).
    
    Next, focusing on Line \ref{alg-line:zkdl-sumcheck}, we consider the case where the prover begins with a false assertion regarding at least one \(\me{\mathbf{X}_k}{\mathbf{u}_k, \mathbf{u}}\) in the sumcheck of \eqref{eq:sc-reindex} but ultimately claims accuracy for the stacked tensor \(\mathbf{X}\). Here, the randomness of \(r_k\)s implies that \eqref{eq:sc-reindex} is likely not satisfied with a probability of \(1 - \frac{1}{\abs{\F}}\). Given this, the failure probability of the sumcheck on \eqref{eq:sc-reindex} is \(1 - O\left(\frac{\log N}{\abs{\F}}\right)\), making the overall deceptive success probability \(\negl{\lambda}\).
    
    Lastly, for any incorrectly computed tensor operation, consider its contribution to the sumcheck protocol \eqref{eq:tensor-op-aggr} observed in Line \ref{alg-line:zkdl-sumcheck}. Consequently, a pair of indices \(\mathbf{n}, \mathbf{i}\) exists for which
    \begin{equation}
        \me{\mathbf{Y}}{\mathbf{n}, \mathbf{i}} \neq \sum_{\mathbf{j} = 0}^{D_\text{in}-1}\me{\beta}{\mathbf{u}, \mathbf{i}} f\left(\me{\mathbf{X}_k}{\mathbf{n}, \mathbf{i}_{I_k}, \mathbf{j}_{J_k}}\right)_{k=0}^{K-1}.
    \end{equation}
    By applying the Schwartz-Zippel Lemma, it is deduced that the equality in \eqref{eq:tensor-op-aggr} cannot be maintained with a probability of \(1-O\left(\frac{\log\left(ND_\text{in}\right)}{\abs{\F}}\right)\). In such scenarios, the sumcheck has an upper bound success rate of \(O\left(\frac{\log\left(ND_\text{in}D_\text{out}\right)}{\abs{\F}}\right)\). Thus, the probability that a deceitful prover triumphs during the sumcheck for combined tensor operations is similarly \(\negl{\lambda}\).
\end{proof} 

\subsection{Zero-knowledge} \label{appendix:zero-knowledge}

Theorem \ref{thm:zero-knowledge} captures the zero-knowledge properties of zkDL, i.e., the execution of zkDL leaks no information about the training data and model parameters:

\begin{Thm}[Zero-knowledge]\label{thm:zero-knowledge}
    Assuming the implementation of the zero-knowledge Pedersen commitment scheme and the zero-knowledge variant of the sumcheck protocol \cite{DBLP:journals/eccc/ChiesaFS17, libra, orion}, Protocol \ref{protocol:zkdl} is zero-knowledge. Specifically, there exists a simulator \(\mathcal{S} = (\mathcal{S}_1, \mathcal{S}_2)\) for which the ensuing two views are computationally indistinguishable by any probabilistic polynomial-time (PPT) algorithm, given the public parameters $\tt{pp}$ (the generators used in the commitment scheme within the context of zkDL):

    \begin{mdframed}[backgroundcolor=white,linewidth=1pt,roundcorner=5pt]
    \noindent\textbf{Real:}
      \begin{algorithmic}[1]
        \State \(\com{\,} \leftarrow \tt{zkDL-Commit}\left(\mathbf{X} \|\mathbf{y}\|\mathbf{W}_\text{init}; \tt{pp}\right)\)
        \State \(\pi \leftarrow \tt{zkDL-Prove}(\tt{com}; \tt{pp})\)
        \State \Return \(\com{\,}, \pi\)
      \end{algorithmic}
    \end{mdframed}
    
    \begin{mdframed}[backgroundcolor=white,linewidth=1pt,roundcorner=5pt]
    \noindent\textbf{Ideal:}
      \begin{algorithmic}[1]
        \State \(\com{\,} \leftarrow \mathcal{S}_1\left(1^\lambda; \tt{pp}\right)\)
        \State \(\pi \leftarrow \mathcal{S}_2(\com{\,}; \tt{pp})\), with oracle access to the correctness of the training procedure
        \State \Return \(\com{\,}, \pi\)
      \end{algorithmic}
    \end{mdframed}
    
    In the aforementioned setting, \(\tt{zkDL-Commit}\left(\mathbf{X} \|\mathbf{y}\|\mathbf{W}_\text{init}; \tt{pp}\right)\) pertains to the steps of training (Line \ref{alg-line:zkdl-train}) and making commitments to the tensors (Line \ref{alg-line:zkdl-commit}) in Protocol \ref{protocol:zkdl}. Meanwhile, \(\tt{zkDL-Prove}(\tt{com}; \tt{pp})\) refers to the processes of the sumcheck protocols (Lines \ref{alg-line:zkdl-sumcheck} and \ref{alg-line:zkdl-reindex}) and the proof of evaluations relative to the commitments (Line \ref{alg-line:zkdl-open}).
\end{Thm}

\begin{proof}[Proof sketch of Theorem \ref{thm:zero-knowledge}] 
    Firstly, by leveraging the zero-knowledge sumcheck protocols, for each tensor operation \(\mathcal{A}_j\), there is a simulator \(\mathcal{S}^{\mathcal{A}_j}\) that indistinguishably simulates the execution of the sumcheck protocol with sole reliance on oracle access to validate the correctness of the aggregated operations among them. In contrast, for each tensor class \(\mathcal{T}_i\), there exists a pair of simulators, denoted as \(\left(\mathcal{S}^{\mathcal{T}_i}_1, \mathcal{S}^{\mathcal{T}_i}_2\right)\). Specifically, \(\mathcal{S}_1^{\mathcal{T}_i}\) is tasked with simulating the generation of the commitment, while \(\mathcal{S}_2^{\mathcal{T}_i}\) simulates the proof of evaluation, contingent on oracle access to ascertain the evaluation at the precise point the committed tensor is assessed.

    Building on this, given that the randomness across all aforementioned simulators remains independent, the overall simulator for zkDL can be architected in the following manner:
    
    \begin{enumerate}
        \item \(\mathcal{S}_1\), realized as a composite of all \(\mathcal{S}_1^{\mathcal{T}_i}\)s, is responsible for simulating the creation of commitments.
        \item With oracle access to the integrity of the comprehensive training process, that is, the correctness of all aggregated tensor operations, \(\mathcal{S}_2\) emerges as a combined entity of both \(\mathcal{S}^{\mathcal{A}_j}\)s and \(\mathcal{S}^{\mathcal{T}_i}_2\)s. Its role is to simulate proofs of tensor operations' correctness through sumchecks and the subsequent proofs of evaluations.
    \end{enumerate}
    
    As a consequence, the composite simulated transcript, owing to the inherent independence between the components generated by these simulators, remains computationally indistinguishable from an authentic transcript.
\end{proof}

\end{document}